\documentclass[11pt]{article}
\pdfoutput=1
\usepackage{enumerate}
\usepackage{pdfsync}
\usepackage[OT1]{fontenc}
\usepackage{smile}
\usepackage[colorlinks,
            linkcolor=red,
            anchorcolor=blue,
            citecolor=blue
            ]{hyperref}
\usepackage{color, colortbl}
\usepackage{fullpage}
\usepackage{hyperref}
\usepackage[protrusion=true,expansion=true]{microtype}
\usepackage{pbox}
\usepackage{setspace}
\usepackage{tabularx}
\usepackage{float}
\usepackage{wrapfig,lipsum}
\usepackage{enumitem}
\usepackage{microtype}
\usepackage{graphicx}
\usepackage{subfigure}
\usepackage{enumerate}
\usepackage{enumitem}
\usepackage{pgfplots}

\usetikzlibrary{arrows,shapes,snakes,automata,backgrounds,petri}
\usepackage{booktabs} 

\allowdisplaybreaks

\newcommand{\event}{\mathcal{E}}

\usepackage{enumitem}

\def \pnorm {B}

\def \alg {\text{APPO}}
\def \method {\text{ADPO}}

\usepackage{colortbl}
\definecolor{LightCyan}{rgb}{0.8, 0.9, 1}
\usepackage[utf8]{inputenc} 
\usepackage[T1]{fontenc}    
\usepackage{hyperref}       
\usepackage{url}            
\usepackage{booktabs}       
\usepackage{amsfonts}       
\usepackage{nicefrac}       
\usepackage{microtype}      
\usepackage{xcolor}         

\usepackage{todonotes}
\usepackage{colortbl}
\definecolor{LightCyan}{rgb}{0.8, 0.9, 1}

\usepackage{enumitem}

\makeatletter
\newcommand*{\rom}[1]{\expandafter\@slowromancap\romannumeral #1@}
\makeatother
\title{\huge Reinforcement Learning from Human Feedback with Active Queries}

\author
{
        Kaixuan Ji\thanks{Equal Contribution}
    \thanks{Department of Computer Science, University of California, Los Angeles, CA 90095, USA; e-mail: {\tt kaixuanji@cs.ucla.edu}}
        ~~~and~~~
	Jiafan He$^*$\thanks{Department of Computer Science, University of California, Los Angeles, CA 90095, USA; e-mail: {\tt jiafanhe19@ucla.edu}}
	~~~and~~~
	Quanquan Gu\thanks{Department of Computer Science, University of California, Los Angeles, CA 90095, USA; e-mail: {\tt qgu@cs.ucla.edu}}
}

\begin{document}
\date{}
\maketitle

\begin{abstract}

Aligning large language models (LLM) with human preference plays a key role in building modern generative models and can be achieved by reinforcement learning from human feedback (RLHF). Despite their superior performance, current RLHF approaches often require a large amount of human-labelled preference data, which is expensive to collect. In this paper, inspired by the success of active learning, we address this problem by proposing query-efficient RLHF methods. We first formalize the alignment problem as a contextual dueling bandit problem and design an active-query-based proximal policy optimization ($\alg$) algorithm with an $\tilde{O}(d^2/\Delta)$ instance-dependent regret bound and an $\tilde{O}(d^2/\Delta^2)$ query complexity, where $d$ is the dimension of feature space and $\Delta$ is the sub-optimality gap over all the contexts. We then propose $\method$, a practical version of our algorithm based on direct preference optimization (DPO) and apply it to fine-tuning LLMs. Our experiments show that $\method$, while only making about half of queries for human preference, matches the performance of the state-of-the-art DPO method.

\end{abstract}

\section{Introduction}

Recent breakthroughs in large language models (LLM) significantly enhance the performances across a wide range of tasks, including common sense reasoning, world knowledge, reading comprehension and math problem solving \citep{jiang2023mistral, touvron2023llama, vicuna2023, tunstall2023zephyr}. In addition to the prominent capabilities of traditional natural language tasks \citep{gao2023exploring, yuan2023zero, han2023information, wei2023zero}, they also demonstrate great potential in responding to human instructions \citep{ouyang2022training}. One key step towards building these models is aligning them with human preference, where reinforcement learning from human feedback (RLHF) \citep{casper2023open, ouyang2022training, ziegler2019fine, christiano2017deep, rafailov2024direct} is widely employed. 
The orthodox process of RLHF \citep{gao2023scaling,munos2023nash} is described as follows. At each time, the human user prompts the LLM with an instruction. Subsequently, the model generates several candidate responses and queries the users for their preferences. Then, a reward model is trained on this preference data to mimic human evaluation.  
The language models are then updated using reinforcement learning (RL) algorithms such as Proximal Policy Optimization (PPO) \citep{schulman2017proximal} to optimize responses that maximize the reward.
However, PPO requires an additional reward model and online sampling from LLMs, which is computational inefficient. Alternatively, Direct Preference Optimization (DPO) \citep{rafailov2024direct} directly treats the language models themselves as the reward models and optimize the LLMs on the offline datasets. While its objective is mathematically equivalent to its canonical counterpart, it eliminates the requirement of additional reward modeling and online sampling.

Despite the notable success of RLHF in aligning language models with human preferences, its practical implementation often necessitates significant amounts of human-labeled preference data. For instance, the fine-tuning process of \texttt{zephyr-7b-beta} through RLHF relies on the utilization of a sizable 62k UltraCat-binarized dataset \citep{ding2023enhancing}. The collection of such a substantial volume of human preference data is both costly and inefficient. Therefore, there exists a pressing need to develop query-efficient RLHF methods for effectively aligning large language models with human preferences. 

Following recent theoretical advancements in RLHF \citep{xiong2023gibbs, zhu2023principled, sekhari2024contextual}, we formulate the RLHF problem as a contextual dueling bandit problem \citep{yue2012k, wu2016double, saha2021optimal, saha2022efficient, saha2022versatile, wu2023borda, di2023variance}. In this setting, the learner proposes a pair of actions and receives noisy feedback regarding the preference between the dueling pair for each round. 
While numerous studies address regret minimization in dueling bandits, only a few works \citet{wang2023rlhf,zhan2023query,wu2023making,sekhari2024contextual} have considered query complexity during the learning process. However, their results either exhibit a linear dependency on the size of the action set $\cA$, limiting the practical applicability of their methods, or fail to provide instance-dependent regret, thereby missing the opportunity to exploit favorable large-suboptimal-gap structures in RLHF.\footnote{For instance, in the Ultrafeedback-binarized dataset, the minimal reward gap is 0.5, within a range of 0 to 10. This occurs because users typically cannot perceive subtle quality differences between responses.}

In this paper, we adopt the principles of active learning \citep{zhang2000value, hoi2006batch} to design a query-efficient algorithm, \textbf{A}ctive \textbf{P}roximal \textbf{P}olicy \textbf{O}ptimization ($\alg$) for linear contextual dueling bandits. In each round, $\alg$ employs the maximum likelihood estimator~\citep{di2023variance} to estimate the underlying parameter and constructs an optimistic estimator for the reward gap between different arms. Subsequently, $\alg$ selects the best arms and estimates the uncertainty associated with the potential feedback. To reduce the query complexity, $\alg$ selectively queries the dueling preference and updates the parameters only when the uncertainty of the observation exceeds a threshold. 

We further extend $\alg$ to direct preference optimization (DPO) \citep{rafailov2024direct} and introduce a novel query-efficient method, \textbf{A}ctive \textbf{D}irect \textbf{P}olicy \textbf{O}ptimization ($\method$). Following the methodology of $\alg$, $\method$ selectively queries human preference only for data where the model exhibits high uncertainty about the observation. For data where the model is less uncertain about,  we employ the pseudo label predicted by the model to fine-tune the model itself. Our contributions are summarized as follows.

\begin{itemize}[leftmargin=*]
\setlength\itemsep{0.5em}
    \item We propose an active-learning based algorithm $\alg$ for linear contextual dueling bandits. Theoretical analysis shows that our algorithm enjoys a constant instance-dependent regret $\tilde{O}(d^2/\Delta)$\footnote{We use notation $\tilde \cO(\cdot)$ to hide the log factor other than number of rounds $T$}, where $d$ is the dimension of the feature space, and $\Delta$ is the minimal sub-optimal gap. Meanwhile, our proposed algorithm only requires $\tilde{O}(d^2/\Delta^2)$ queries in total $T$ rounds. Compared with previous instance-dependent regret bound $\tilde{O}( A^2 \beta^2 d/\Delta )$ achieved by \citet{sekhari2024contextual}\footnote{In our work, we only focused on the regret of one selected action, which is slightly different from the two-arm regret in \citet{sekhari2024contextual}. See Section~\ref{sec:pre} for further discussion.}, where $A$ is the size of the action space, our regret bound is independent on the size of action space $A$, which is more favorable in practice.
    \item We propose an active-learning-based DPO method named $\method$. We apply our method to train \texttt{zephyr-7b-beta} on Ultrafeedback-binarized dataset \citep{ding2023enhancing} and \texttt{zephyr-7b-gemma} on dpo-mix-7k dataset. Our experiment shows that while $\method$ only make less than half numbers of queries, the model trained by $\method$ achieves a comparable or better performance than DPO on our selected benckmarks including MT-Bench~\citep{zheng2024judging} and AlpacaEval 2.0.
\end{itemize}

\paragraph{Notation } 
We employ $[n]$ to denote the set $\{1,\dots, n\}$. In this work, we use lowercase letters to represent scalars, and denote vectors and matrices by lower and uppercase boldface letters respectively. Given a vector $\xb\in \RR^d$, we denote the vector's $L_2$-norm by $\|\xb\|_2$. We further define $\|\xb\|_{\bSigma}=\sqrt{\xb^\top\bSigma\xb}$ given a positive semidefinite matrix $\bSigma \in \RR^{d\times d}$. We use standard asymptotic notations $O(\cdot), \Omega(\cdot), \Theta(\cdot)$, and further use $\tilde O(\cdot)$ to hide logarithmic factors other than the number of rounds $T$. We use $\ind\{\cdot\}$ denote the indicator function.

\section{Related Work}


\paragraph{Reinforcement Learning from Human Feedback } Learning from human preference data dates back to \citet{wirth2017survey, christiano2017deep} and is recently popularized by generative language models \citep{achiam2023gpt, touvron2023llama}. This procedure usually takes place after supervised finetuning (SFT).
The canonical procedure of aligning with human preference includes two stages: reward modeling and reward maximization \citep{ouyang2022training, bai2022training, munos2023nash}. Another approach is direct preference optimization (DPO) \citep{rafailov2024direct}, which treats the generative models directly as reward models and trains them on preference data. Compared with the first approach, DPO simplifies the aligning process while maintaining its effectiveness. 
However, both paradigms require a large amount of human preference data. 
In this work, we follow the DPO approach and study its query-efficient modification. The empirical success of RLHF also prompts a series of theoretical works, with a predominant focus on the reward maximization stage, modeling this process as learning a dueling bandit \citep{zhu2023principled, xiong2023gibbs, sekhari2024contextual}. Among these works, \citet{wang2023rlhf,zhan2023query,wu2023making,sekhari2024contextual} stand out for considering query complexity in the learning process. However, \citet{wang2023rlhf,zhan2023query,wu2023making} focus either on worst-case regret bounds or the sample complexity to identify an $\epsilon$-optimal policy, but fail to provide instance-dependent guarantees (Further discussion is deferred to Appendix  \ref{app:discussion}). Only \citet{sekhari2024contextual} offer an instance-dependent analysis; however, their upper bound is $\tilde{O}(A^2 \beta^2 d/\Delta)$, which depends on the size of the action set $A$, limiting the practical applicability of their algorithm. Compared to this work, we provide an instance-dependent regret guarantee without dependency on the action space. Furthermore, based on APPO, we derive a practical algorithm, ADPO, which we empirically verify to demonstrate its superiority. 
We also notice two concurrent works that incorporate active learning with DPO. \citet{mehta2023sample} incorporate active learning to DPO and use the variance of log-probabilities under different dropouts as the uncertainty estimator, which is inefficient in practice. 
\citet{muldrew2024active} also proposed an active learning-based alternative of DPO and leverage reward difference as uncertainty estimator. However, their approach does not involve pseudo labels, which is a key component of our approach. 

\paragraph{Dueling Bandits }
Dueling bandits represent a variant of the multi-armed bandit problem, incorporating preference feedback between two selected arms \citep{yue2012k}. Existing results in this domain generally fall into two categories, shaped by their assumptions about preference probability. The first category of work \citep{yue2012k, falahatgar2017maximum, falahatgar2018limits, ren2019sample, wu2022adaptive, lou2022active} assumes a transitivity property for preference probability and focuses on identifying the optimal action. Our work also belongs to this category. The second category of work \citep{jamieson2015sparse, heckel2018approximate, saha2021optimal, wu2023borda, dudik2015contextual, ramamohan2016dueling, balsubramani2016instance} focuses on general preferences with various criteria for optimal actions, such as Borda winner and Copeland winner.

Expanding beyond the standard dueling bandit problem, \citet{dudik2015contextual} was the first to incorporate contextual information into the dueling bandit framework. Subsequently, \citet{saha2021optimal} studied the $K$-arm contextual dueling bandit problem and proposed an algorithm with a near-optimal regret guarantee. In order to addressing the challenge of a potentially large action space, \citet{bengs2022stochastic} also considered linear function approximation and extended these results to the contextual linear dueling bandit problem and obtained a regret guarantee of $\tilde O(d\sqrt{T})$. Recently, \citet{di2023variance} introduced a layered algorithm, improving the results to a variance-aware guarantee of $\tilde O(d\sqrt{\sum \sigma_t^2})$, where $\sigma_t^2$ denotes the variance of the observed preference in round $t$.

\paragraph{Active Learning } To mitigate the curse of label complexity, active learning serves as a valuable approach in supervised learning . The first line of work is pool-based active learning \citep{zhang2000value,hoi2006batch,gu2012selective,gu2014batch,citovsky2021batch}. In pool-based active learning, instead of acquiring labels for the entire dataset, the learner strategically selects a batch of the most informative data at each step and exclusively queries labels for this selected data batch. The learner then employs this labeled data batch to update the model. Subsequently, guided by the enhanced model, the learner queries another mini-batch of labels and continues the training process. These steps are iteratively repeated until the model achieves the desired performance level. The strategic selection of informative data significantly reduces the label complexity for supervised learning. The label complexity of pool-based active learning has been extensively studied by \citet{dasgupta2005coarse,dasgupta2005analysis,balcan2006agnostic,balcan2007margin,hanneke2015minimax,gentile2022achieving}.
On the other hand, selective sampling (a.k.a., online active learning) \citep{cesa2005minimizing,cesa2006worst,cesa2009robust,hanneke2021toward} is a learning framework that integrates online learning and active learning. In this framework, the algorithm sequentially observes different examples and determines whether to collect the label for the observed example.
In reinforcement learning, there are also lines of works focusing on the application of active learning. On theoretical side, \citet{schulze2018active, krueger2020active, tucker2023bandits} focuses on active reinforcement learning and directly integrates the query cost into the received reward. \citet{krueger2020active} laid the groundwork for active reinforcement learning by introducing a cost $c$ associated with each reward observation and evaluated various heuristic algorithms for active reinforcement learning. Recently, \citet{tucker2023bandits} studied the multi-arm bandit problem with costly reward observation. Their work not only suggests empirical advantages but also proves an $O(T^{2/3})$ regret guarantee. On the application side, there are also lines of works apply selective sampling to specific circumstance of RLHF in robotics~ \citep{lee2021b,lee2021pebble,liang2022reward}.

\section{Preliminaries}\label{sec:pre}

In this work, we formulate the RLHF problem as a contextual dueling bandit problem \citep{saha2021optimal,di2023variance}. We assume a context set $\cX$, and at the beginning of each round, a contextual variable $x_t$ is i.i.d generated from the context set $\cX$ with the distribution $\cD$. Based on the context $x_t$, the learner then chooses two actions $y_t^1, y_t^2$ from the action space $\cA$ and determines whether to query the environment for preferences between these actions. If affirmative, the environment generates the preference feedback $o_t$ with the following probability $\PP(o_t=1|x_t, y_t^1, y_t^2) = \sigma\big(r(x_t, y_t^1) - r(x_t, y_t^2)\big)$, where $\sigma(\cdot): \mathbb{R} \rightarrow [0,1]$ is the link function and $r(\cdot, \cdot)$ is the reward model.

We consider a linear reward model, e.g., $r(x, y) = \langle \btheta^*, \bphi(x,y)\rangle$, where $\btheta^* \in \RR^d$ and $\bphi: \cX \times \cA \rightarrow \RR^d$ is a known feature mapping. For the sake of simplicity, we use $\bphi_t^1,\bphi_t^2$ to denote $\bphi(x_t, y_t^1),\bphi(x_t,y_t^2)$. Additionally, we assume the norm of the feature mapping $\bphi$ and the underlying vector $\btheta^*$ are bounded.
\begin{assumption}\label{ass:setup}
The linear contextual dueling bandit satisfies the following conditions:
\begin{itemize}[leftmargin=*]
    \item For any contextual $x\in \cX$ and action $y \in \cA$, we have $\|\bphi(x,y)\|_2\leq L/2$ and $r(x, y) \leq 1$.\
    \item For the unknown environment parameter $\btheta^*$, it satisfies  $\|\btheta^*\|_2\leq B$.
\end{itemize}
\end{assumption}
For the link function $\sigma$, we make the following assumption, which is commonly employed in the study of generalized linear contextual bandits \citep{filippi2010parametric, di2023variance}.
\begin{assumption}\label{assumption:differ}
    The link function $\sigma$ is differentiable and the corresponding first derivative satisfied $\kappa_{\sigma}\leq \dot{\sigma}(\cdot)$,
    where $\kappa_{\sigma}>0$ is a known constant.
\end{assumption}
 The learning objective is to minimize the cumulative regret defined as:
\begin{align*}
    \text{Regret}(T) = \sum_{t=1}^T r^*(x_t) - r(x_t, y_t^1),
\end{align*}
where $r^*(x_t)=r^*(x_t,y_t^*)=\max_{y\in \cA} r^*(x_t,y)$ stands for the largest possible reward in context $x_t$. 
It is worth noting that prior works \citep{di2023variance,saha2022efficient,sekhari2024contextual} in dueling bandits often define the regret on both action $y_t^1$ and $y_t^2$. However, in the context of RLHF, the model generates multiple candidate responses, and users will choose the most preferable response from the available options. Under this circumstance, sub-optimality is only associated with the selected response. Therefore, we choose the regret defined only on action $y_t^1$.

To quantify the cost of collecting human-labeled data, we introduce the concept of query complexity $\text{Query}(T)$ for an algorithm, which is the total number of data pairs that require human feedback for preference across the first $T$ rounds. 
Note that while some prior work~\citep{tucker2023bandits} counts the cost of requesting for human feedback together with the cost paid for taking certain action, 
in our approach, we distinguish between regret and query complexity as two separate performance metrics for an algorithm.

In addition, we consider the minimal sub-optimality gap    \citep{simchowitz2019non,yang2020q,he2021logarithmic}, which characterizes the difficulty of the bandit problem.
\begin{definition}[Minimal sub-optimality gap]
For each context $x\in\cX$ and action $y\in\cA$, the sub-optimality gap $\Delta(x,y)$ and the minimal gap $\Delta$ are defined as
\begin{align*}   
    \Delta(x,y)=r^*(x)-r(x,y), \quad \Delta=\min_{x\in \cX,y\in \cA}\big\{\Delta(x,y): \Delta(x,y)\ne 0\big\}.
\end{align*}
\end{definition}
In general, a larger sub-optimality gap $\Delta$ between action $y$ and the optimal action $y^*$ implies that it is easier to distinguish between these actions and results in a lower cumulative regret. Conversely, a task with a smaller gap $\Delta$ indicates that it is more challenging to make such a distinction, leading to a larger regret.
\noindent In this paper, we assume the minimal sub-optimality gap is strictly positive.
\begin{assumption}\label{assumption:gap}
The minimal sub-optimality gap is strictly positive, i.e., $\Delta>0$.
\end{assumption}

\begin{remark}
In the context of RLHF, given a prompt, the minimal sub-optimality gap $\Delta$ represents the uniform gap between the best answers and the sub-optimal answers, where the optimal answers (or arms in the context of bandits) might not be unique (See Definition 3.3). Typically, for arms with sub-optimality close to 0, it is difficult for humans to discern a quality difference between them. Under this situation, we can roughly consider these arms also as optimal arms (optimal arms may not be unique) and only consider the gap between sub-optimal arms and these optimal arms. Therefore, it is realistic to make such an assumption in the context of RLHF.
\end{remark}

\begin{algorithm}[ht]
    \caption{Active Proximal Policy Optimization ($\alg$)}\label{algo:po}
    \begin{algorithmic}[1]
    \REQUIRE Regularization parameter $\lambda>0$, 
    and $\pnorm$, 
    an upper bound on the $\ell_2$-norm of $\btheta^*$, confidence radius $\beta$, uncertainty threshold $\Gamma>0$, learning rate $\eta$
    \STATE Set  initial policy $\pi_1(\cdot | \cdot)$ as uniform distribution over the action set $\cA$, $\bSigma_0 \leftarrow \lambda \Ib$, $\cC_0 = \emptyset$
    \FOR{$t=1,\ldots, T$}
        \STATE Compute the MLE $\hat{\btheta}_t$ as in \eqref{eq:mle} and observe $\cA$, select $y_t^2 \sim \text{Uniform}(\cA)$
        \STATE Compute $\hat{D}_t(x_t, y) = \min \bigl\{ \langle \hat{\btheta}_t, \bphi(x_t,y)-\bphi^2_t \rangle + \beta \|\bphi(x_t,y) - \bphi^2_t\|_{\bSigma_{t-1}^{-1}}, 1 \bigr\}$
        \STATE Choose $y^1_t = \argmax_{y} \hat{D}_t(x_t, y)$
        \IF {$\| \bphi^1_t - \bphi^2_t \|_{\bSigma_{t-1}^{-1}} \leq \Gamma$} \label{line:6}
            \STATE Keep $\bSigma_t = \bSigma_{t-1}$, $\pi_{t+1}(a|s)=\pi_{t}(a|s)$ and $\cC_t = \cC_{t-1}$
        \ELSE
            \STATE Sample $y^1_t \sim \pi_t(\cdot | s_t)$, query for the preference and observe $o_t$\label{line:10}
            \STATE Update $\bSigma_t = \bSigma_{t-1} + (\bphi^1_t-\bphi^2_t)(\bphi^1_t-\bphi^2_t)^\top$ and $\cC_t = \cC_{t-1} \cup \{t\}$
            \STATE Update $\pi_{t+1}(y|x) \propto \pi_{t}(y|x)\exp\big(\eta \hat{D}_t(y,x)\big)$\label{line:oppo}
        \ENDIF
    \ENDFOR
    \end{algorithmic}
\end{algorithm}

\section{Algorithm}

In this section, we introduce our proposed query-efficient method for aligning LLMs. 
The main algorithm is illustrated in Algorithm~\ref{algo:po}. At a high level, the algorithm leverages the uncertainty-aware query criterion \citep{zhang2023interplay} to issue queries and employs Optimistic Proximal Policy Optimization (OPPO)~\citep{cai2020provably,he2022near} for policy updates.
In the sequel, we introduce the key parts of the proposed algorithm.

\paragraph{Regularized MLE Estimator }
For each round $t\in[T]$, we construct the regularized MLE estimator \citep{filippi2010parametric,li2017provably} of parameter $\btheta^*$ by solving the following equation:
\begin{align}\label{eq:mle}
    \lambda \btheta + \textstyle{\sum_{\tau \in \cC_{t-1}}}\big[o_{\tau} - \sigma\big(\langle \btheta, \bphi^1_{\tau}-\bphi^2_{\tau} \rangle\big)\big] (\bphi^1_{\tau}-\bphi^2_{\tau}) = \mathbf{0},
\end{align}
where $\cC_t$ denotes the set of rounds up to the $t$-th round for which the preference label is required.
Compared with previous work on linear dueling bandits \citep{saha2021optimal,di2023variance}, here we only requires part of the human-labelled preference. We construct the MLE estimator with only rounds $\tau \in \cC_t$. In addition, the estimation error between $\hat{\btheta}_t$ and $\btheta^*$ satisfies 
\begin{align*}
    \|\btheta^* - \hat{\btheta}_t \|_{\bSigma_{t-1}} \leq \tilde O\big(\sqrt{d\log |\cC_t|}/\kappa_{\sigma}\big).
\end{align*}
After constructing the estimator $\hat{\btheta}_t$, the agent first selects a baseline action $y^2_t$ and compares each action $y \in \cA$ with the baseline action $y_t^2$. For simplicity, we denote $D_t(x_t, y) = \langle \btheta^*, \bphi(x_t,y)-\bphi^2_t \rangle$ as the reward gap between $y$ and action $y_t^2$. Then, we construct an optimistic estimator $\hat{D}_t$ for the reward gap with linear function approximation and Upper Confidence Bound (UCB) bonus, i.e.,
\begin{align*}
    \hat{D}_t(x_t,y) = \min \{\langle \hat{\btheta}_t, \bphi(x_t,y)-\bphi^2_t \rangle + \beta \|\bphi(x_t,y) - \bphi^2_t\|_{\bSigma_{t-1}^{-1}}, 1\}.
\end{align*}
Here we truncate the estimation since the true reward is in $[0,1]$ and therefore their difference is bounded by $1$. With the help of UCB bonus, we can show that our estimated reward gap $\hat{D}_t$ is an upper bound of the true reward gap $D_t$.

\paragraph{Uncertainty-Aware Query Criterion }
To mitigate the expensive costs from collecting human feedback, we introduce the uncertainty-based criterion (Line \ref{line:6}) \citep{zhang2023interplay} to decide whether a pair of action requires $y_t^1$ and $y_t^2$ requires human-labelled preference.
Intuitively speaking, the UCB bonus $\beta \|\bphi_t^1-\bphi_t^2\|_{\bSigma_{t-1}^{-1}}$ captures the uncertainty associated with the preference feedback $o_t$.
Similar criterion has also been used in corruption-robust linear contextual bandits \citep{he2022nearly} and nearly minimax optimal algorithms for learning linear (mixture) Markov decision processes \citep{zhou2022computationally,he2023nearly,zhao2023variance}, where $\beta\|\bphi\|_{\bSigma_{t-1}^{-1}}$ represents the uncertainty of certain action.
For the action pair $(y_t^1,y_t^2)$ with low uncertainty, where the observation is nearly known and provides minimal information, we select the action $y_t^1,y_t^2$ without querying human preference feedback. In this situation, the policy $\pi(\cdot|\cdot)$ remains unchanged as there is no observation in this round. By employing the uncertainty-based data selection rule, we will later prove that the query complexity is bounded. 

\paragraph{Proximal Policy Optimization }
In cases where the action pair $(y_t^1, y_t^2)$ exhibits high uncertainty and the uncertainty-aware query criterion is triggered, the agent resample the action $y_t^1$ from policy $\pi_t$ and queries human feedback for the duel $y_t^1, y_t^2$. Upon observing the preference $o_t$, this round is then added to the dataset $\cC_t$. Subsequently, the policy $\pi_{t+1}$ is updated using the Optimistic Proximal Policy Optimization (OPPO) method \citep{cai2020provably,he2022near}, i.e.,
\begin{align*}
    \pi_{t+1}(y|x) \propto \pi_{t}(y|x)\exp\big(\eta \hat{D}_t(y,x)\big).
\end{align*}
In an extreme case where the uncertainty threshold $\Gamma$ is chosen to be $0$, the uncertainty-aware query criterion will always be triggered. Under this situation, Algorithm~\ref{algo:po} will query the human-labeled preference for each duel $(y_t^1,y_t^2)$, and Algorithm~\ref{algo:po} will degenerate to the dueling bandit version of OPPO \citep{cai2020provably}. Under this situation, Algorithm~\ref{algo:po} enjoys $\tilde O(d\sqrt{T})$ regret while having a linear query complexity with respect to the number of rounds $T$.

\section{Theoretical Analysis}\label{sec:main-result}
In this section, we present our main theoretical results.

\begin{theorem}\label{thm:main}
Let $\Delta$ be the minimal sub-optimal gap in Assumption \ref{assumption:gap}. If we set the parameters $\Gamma= \tilde{O}(\Delta/\sqrt{{d}})$, $\lambda=B^{-2}$, $\eta=\tilde O(\sqrt{\Gamma^2 \log \cA /d})$,  and $\beta=\tilde O(\sqrt{d}/\kappa_{\sigma})$ in Algorithm~\ref{algo:po}, then with with probability at least $1-\delta$, the regret for Algorithm~\ref{algo:po} across the first $T$ rounds is upper bounded by
\begin{align*}
    \text{Regret}(T) = \tilde{O}(d^2/\Delta).
\end{align*}
In addition, the query complexity of Algorithm~\ref{algo:po} is upper bounded by:
\begin{align*}
   \text{Query}(T) =  |\cC_T| = \tilde{O}(d^2/\Delta^2).
\end{align*}
\end{theorem}
\begin{remark}
    Theorem~\ref{thm:main} suggests that our algorithm achieves a constant level of regret and query complexity respect to the number of rounds $T$. In theory, our algorithm requires a prior knowledge of the sub-optimal gap $\Delta$. In practice where $\Delta$ is unknown, the learner can set the parameter $\Delta$ via grid search process.
\end{remark}
\begin{remark}
In comparison to the instance-dependent regret $\tilde O(A^2d/\Delta)$ obtained by the AURORA algorithm \citep{sekhari2024contextual}\footnote{In our work, we only focused on the regret of one selected action, which slightly differs from the regret in \citet{sekhari2024contextual}.}, our algorithm's regret eliminates the dependency of the action space $A$. Moreover, we achieve an improvement in the query complexity by a factor of $A^3$. 
\end{remark}

\section{Practical Algorithm}

\begin{algorithm}[t]
    \caption{Active Direct Preference Optimization ($\method$)}\label{algo:dpo}
    \begin{algorithmic}[1]
    \REQUIRE Regularization parameter $\beta$, uncertainty threshold $\gamma$, learning rate $\eta$, initial model parameter $\btheta_1$, batch size $S$
    \FOR{$t=1,\ldots, T$}
        \STATE Receive batch of data $\cB_t = \{x_i, y_i^1, y_i^2\}_{i=1}^S$
        \FOR{$i=1,\ldots, S$}
        \STATE Set the confidence $C_{\btheta_t}(x_i, y_i^1, y_i^2)$ as in~\eqref{eq:ucfuction}
        \IF{$C_{\btheta_t}(x_i, y_i^1, y_i^2) \leq \gamma$}
        \STATE Query for the human label and set $o_i$ as the queried preference.
        \ELSE
        \STATE Set $o_i \leftarrow \text{sign}\big(r_{\btheta_t}(x,y^1) - r_{\btheta_t}(x,y^2)\big)$
        \ENDIF
        \ENDFOR
        \STATE Update $\btheta_{t+1} \leftarrow \btheta_t - \eta \nabla_{\btheta} \cL_{\cB_t}(\pi_{\btheta_t}, \pi_{\btheta_1})$
    \ENDFOR
    \end{algorithmic}
\end{algorithm}

In this section, we introduce a practical version of our proposed algorithm based on DPO \citep{rafailov2024direct} and the resulting algorithm is named as Active Direct Preference Optimization ($\method$) and summarized in Algorithm~\ref{algo:dpo}. 
At a high level, our proposed method follows the basic idea of Algorithm~\ref{algo:po} and sets an uncertainty threshold to filter out informative training samples. However, adapting our algorithm to neural network training requires several key modifications as below.

\paragraph{Direct Preference Optimization } We follow the framework of DPO \cite{rafailov2024direct} for policy optimization. In detail, we consider the Bradley-Terry (BT) model \citep{bradley1952rank}, which corresponds to $\sigma(x)=1/(1+e^{-x})$. In RLHF, the objective is to maximize the expected reward regularized by the Kullback-Leibler (KL) divergence from the reference policy $\pi_{\text{ref}}$:
\begin{align}\label{eq:obj-rlhf}
    \max_{\pi} \mathbb{E}_{y\sim \pi(\cdot|x), x \sim \cD}\big[r(x,y) - \beta\mathrm{KL}(\pi(\cdot|x) || \pi_{\text{ref}}(\cdot|x))\big],
\end{align}
where $\beta>0$ is the regularization parameter, $\cD$ is the distribution of the prompts and $\pi_{\text{ref}}$ is the reference policy, which corresponds to the SFT checkpoint. The optimal policy of~\eqref{eq:obj-rlhf} is follows:
\begin{align*}
    \pi^*(y|x) \propto \pi_{\text{ref}}(y|x) \exp(r(x,y)).
\end{align*}
Therefore, given the final model parameter $\btheta$, we can rewrite the reward in the following form:
\begin{align*}
    r_{\btheta}(x,y) = \beta \big(\log \pi_{\btheta}(y|x) - \log \pi_{\text{ref}} (y|x)\big) + \beta Z(x),
\end{align*}
where $Z(x)$ is a constant independent of $y$.
Plugging $r_{\btheta}(x,y)$ into the BT model and fitting the model with the preference labels in dataset $\cD$, we get the following DPO training objective:
\begin{align*}
\cL_{\text{DPO}}(\pi_{\btheta},  \pi_{\text{ref}})  = -\EE_{(x,y^1, y^2, o)\sim \cD} \Big[ \log \sigma \Big( o \cdot\big( r_{\btheta}(x,y^1) - r_{\btheta}(x,y^2) \big) \Big) \Big],
\end{align*}
where $y^1$ and $y^2$ are the two responses to the given prompt $x$, and $o$ is the human preference such that $o=1$ indicates a preference for $y^1$, and $o=-1$ indicates a preference for $y^2$. 
Compared to standard RLHF, DPO bypasses the reward modeling process and thus eliminates the introduced reward noise. 

\paragraph{Confidence Estimator } 
The key to achieving query efficiency in Algorithm~\ref{algo:po} is the confidence-based data filter.
However, in real applications, rewards are no longer necessarily parameterized by a linear function. Thus, the uncertainty estimator cannot be directly transferred to empirical cases. Since the model is essentially predicting the probability of human preference labels, i.e.,
\begin{align}\label{eq:uncertain}
    \mathbb{P}(o=1 | x, y^1, y^2) = \sigma (r_{\btheta}(x, y^1) - r_{\btheta}(x, y^2)),
\end{align}
where $o$ stands for the preference label and $r_{\btheta}$ is the reward model. We can use the reward model's predicted probability as its uncertainty. 
Specifically, if $|r_{\btheta}(x, y^1) - r_{\btheta}(x, y^2|$ is large, then the predicted probability is close to $0$ or $1$, which means the model is confident about its prediction. Otherwise, if $|r_{\btheta}(x, y^1) - r_{\btheta}(x, y^2|$ is close to $0$, the predicted probability is close to $1/2$, which indicates the model's confidence is low. Therefore, we define the following function $C_{\btheta}$:
\begin{align}\label{eq:ucfuction}
    C_{\btheta}(x, y^1, y^2) = |r_{\btheta}(x, y^1) - r_{\btheta}(x, y^2)|,
\end{align}
as the confidence level of the model. 

\paragraph{Training Objectives }
One key design in $\method$ is the use of pseudo label, which is inspired by previous methods such as \citet{gentile2022achieving}. For given answer pairs, if the model is very confident in its preference label, we then use the preference label predicted by the model (i.e., pseudo label) for training. To be specific, given a prompt $x$ and the corresponding answers $y^1$ and $y^2$, the predicted preference label can be defined as follows:
\begin{align} \label{eq:pseudo}
o_{\btheta}(x, y^1, y^2)  = \begin{cases}
  o & \text{if $C_{\btheta}(x, y^1, y^2) \leq \gamma$}\\
  \text{sign}\big(r_{\btheta}(x,y^1) - r_{\btheta}(x,y^2)\big) & \text{if $C_{\btheta}(x, y^1, y^2) > \gamma$}
\end{cases},
\end{align}
where $o$ is the human preference upon query, sign$(z)$ is the signal of $z$ and $\gamma$ is the confidence threshold (corresponding to the threshold $\Gamma$ in $\alg$). With the predicted preference labels of given prompts and answers, now we can formulate our training objective as the follows:
\begin{align}\label{eq:emloss}
    \cL_\cD (\pi_{\btheta},  \pi_{\text{ref}}) = -\EE_{(x,y^1, y^2)\sim \cD} \Big[ \log \sigma \Big( o_{\btheta}(x,y^1,y^2)\cdot \big(r_{\btheta} (x, y^1) - r_{\btheta} (x, y^2)\big) \Big) \Big],
\end{align}
To make our approach more time efficient in practice, we follow the standard approach in DPO and use mini-batch gradient descent to update the parameters of our model. At each time step, we feed the model with a batch of data $\{ (x_i, y^1_i, y^2_i) \}_{i=1}^S$. We then compute the pseudo labels and update the model parameters by one-step gradient descent.

\section{Experiments}

In this section, we conducted extensive experiments to verify the effectiveness of $\method$. Our experiments reveal that $\method$ outperforms DPO while requiring only up to half of the queries. Additionally, our ablation studies show that involving pseudo-labels plays a key role in the training process.

\definecolor{maroon}{cmyk}{0,0.87,0.68,0.32}
\begin{table*}[t]
    \centering
    \caption{Results on objective benchmarks. $\method$ significantly outperforms DPO on ARC, TruthfulQA and performs comparably on HellaSwag, resulting higher average performances. Besides, ADPO only makes 16k queries for Zephyr-Beta and 3k queries for Zephyr-Gemma, which is about only a quarter to half of the queries made by DPO. }
    \vspace*{1pt}
    \begin{tabular}{l | c c c c | c }
    \toprule
        Models & ARC & TruthfulQA  & HellaSwag & Average & \# Queries \\
        \midrule
Zephyr-Beta-SFT      & 58.28  & 40.36   & 80.72 & 59.79  & 0 \\
Zephyr-Beta-DPO  & 61.17 & 45.15  & 82.08 & 62.80 & 62k     \\
\rowcolor{LightCyan} 
Zephyr-Beta-$\method$ & \textbf{62.29} & \textbf{52.25} & \textbf{83.11} & \textbf{65.88} & \textbf{16k} \\
\midrule
Zephyr-Gemma-SFT & 55.03 & 46.92 & 81.45 & 61.13 & 0 \\
Zephyr-Gemma-DPO & 58.45 & 52.07 & \textbf{83.48} & 64.67 & 6751 \\
\rowcolor{LightCyan}  Zephyr-Gemma-$\method$ & \textbf{61.01} & \textbf{57.55} & 83.16 & \textbf{67.24} & \textbf{3652} \\
\bottomrule
\end{tabular}
\vspace*{-5pt}
\label{tab:main-result}
\end{table*}

\begin{table*}[ht!]
\centering
    \caption{Results on subjective benchmarks including AlpacaEval 2.0 and MT-Bench. Here WR stands for win rate and LC stands for length controlled. $\method$ achieves comparable performance with DPO on starting from Zephyr-Beta-SFT and outperforms DPO starting from Zephyr-Gemma-SFT. Besides, ADPO only makes 16k queries for Zephyr-Beta and 3.6k queries for Zephyr-Gemma, which is about only a quarter to half of the queries made by DPO. }
    \vspace*{1pt}
\begin{tabular}{l | c  c  c | c  c  c}
\toprule
\multirow{2}{*}{Models} & \multicolumn{3}{c|}{MT-Bench} & \multicolumn{3}{c}{Alpaca Eval 2.0} \\
 & First Turn & Second Turn & Average & \multicolumn{1}{l}{LC WR} & \multicolumn{1}{l}{WR} & \multicolumn{1}{l}{Avg. Length} \\
 \midrule
Zephyr-Beta-SFT & 6.82 & 5.94 & 6.39 & 4.59 & 4.69 & 1741 \\
Zephyr-Beta-DPO & \textbf{7.55} & \textbf{7.27} & \textbf{7.41} & \textbf{13.57} & \textbf{12.67} & 1735 \\
\rowcolor{LightCyan} Zephyr-Beta-ADPO & 7.31 & 7.08 & 7.20 & 12.67 & 12.02 & 1801 \\
 \midrule
Zephyr-Gemma-SFT & 5.62 & 5.56 & 5.59 & 0.13 & 0.62 & 4296 \\
Zephyr-Gemma-DPO & 5.94 & 5.49 & 5.72 & 10.70 & 3.68 & 9064 \\
\rowcolor{LightCyan} Zephyr-Gemma-ADPO & \textbf{6.53} & \textbf{6.49} & \textbf{6.51} & \textbf{15.85} & \textbf{3.81} & 8967 \\
\bottomrule
\end{tabular}
\vspace*{-3pt}
\label{tab:main-result-alpaca-mt}
\end{table*}

\subsection{Experimental Setup}\label{sec:exp-setup}

\paragraph{Models and Dataset} We start from two different base models \texttt{zephry-7b-sft-full}\footnote{https://huggingface.co/alignment-handbook/zephyr-7b-sft-full} (Zephyr-Beta-SFT) and \texttt{zephyr-7b-gemma-sft-v0.1}\footnote{https://huggingface.co/HuggingFaceH4/zephyr-7b-gemma-sft-v0.1} (Zephyr-Gemma-SFT), which is supervised-finetuned from Mistral-7B \citep{jiang2023mistral} model and gemma-7B~\citep{team2024gemma} correspondingly. Zephyr-Beta-SFT is obtained by conducting SFT on Ultrachat-200k \citep{ding2023enhancing} dataset and Zephyr-Gemma-SFT is obtained by conducting SFT on deita-10k-v0-sft~\citep{liu2023makes}. We follow the approach in alignment-handbook\footnote{https://github.com/huggingface/alignment-handbook} and adopt the corresponding human-preference datasets. Specifically, we use Ultrafeedback-binarized \citep{ding2023enhancing} to train Zephyr-Beta-SFT and dpo-mix-7k\footnote{https://huggingface.co/datasets/argilla/dpo-mix-7k} to train Zephyr-Gemma-SFT.


\paragraph{Baseline and Evaluation } We consider DPO as our baseline and use full-finetune to optimize the models for both DPO and $\method$. Please refer to Appendix~\ref{app:add-exp-setup} for more details regarding the selection of hyperparameters. 
We adopt both objective and subjective evaluation techniques to evaluate the resulting models. Specifically, we employ ARC~\citep{clark2018think}, HellaSwag~\citep{zellers2019hellaswag} and TruthfulQA~\citep{lin2021truthfulqa} as benchmarks for objective evaluation. Among these datasets, ARC~\citep{clark2018think} and HellaSwag~\citep{zellers2019hellaswag} focus on the language models' capability of commonsense reasoning, while TruthfulQA~\citep{lin2021truthfulqa} focuses on human falsehood mimic. For subjective benchmarks, we consider AlpacaEval 2.0 (AlpacaEval) and MT-Bench~\citep{zheng2024judging}. AlpacaEval employs AlpacaFarm \citep{dubois2024alpacafarm}, which is made up of general human instructions, as its set of prompts. During evaluating, the model responses and the reference response generated by GPT-4-Turbo are fed into a GPT-4-Turbo for preference annotation and the win rate measures the models capability. MT-Bench is composed of 80 high-quality multi-turn open-ended questions covering a variety of topics. The generated answers are also judged by GPT-4, which gives scores directly without comparison. Please refer to Appendix~\ref{app:add-exp-setup} for more detailed discussion of the datasets and evaluation. 


\begin{figure*}[ht]
\centering   
\subfigure[ARC Challenge]{\label{fig:main_arc_full}\includegraphics[width=0.32\textwidth]{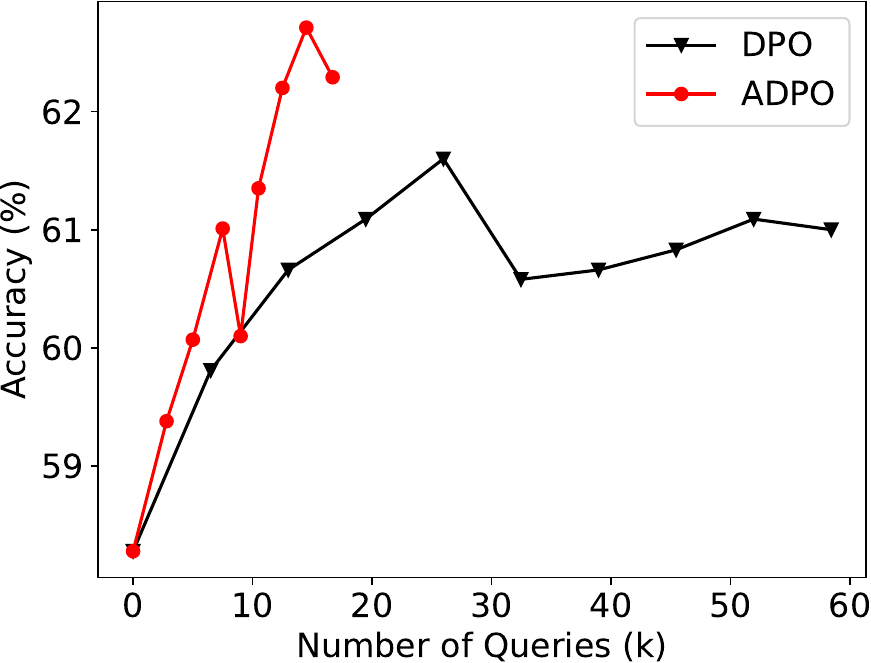}}
\subfigure[TruthfulQA]{\label{fig:main_truth_full}\includegraphics[width=0.32\textwidth]{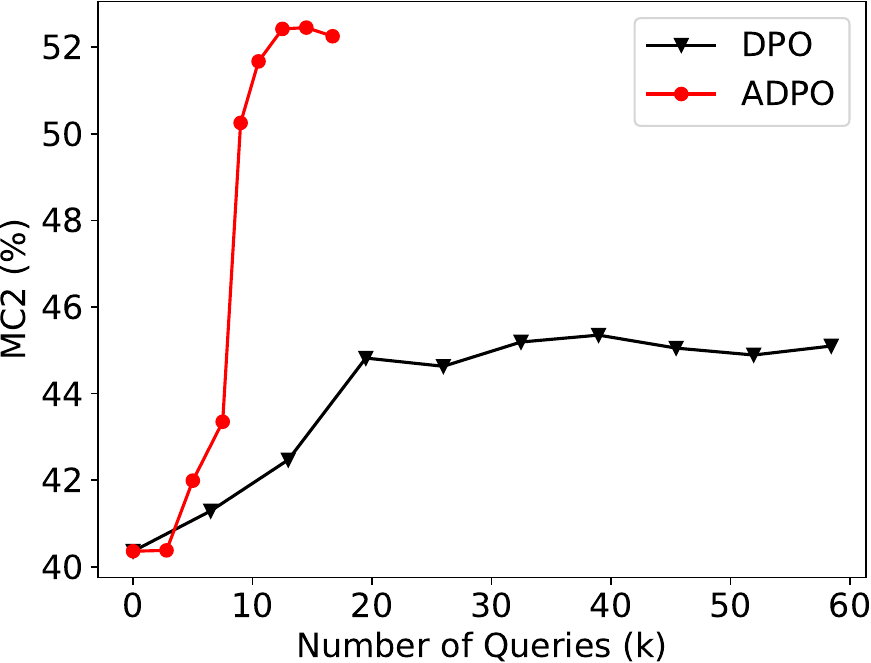}}
\subfigure[HellaSwag]{\label{fig:main_hella_full}\includegraphics[width=0.33\textwidth]{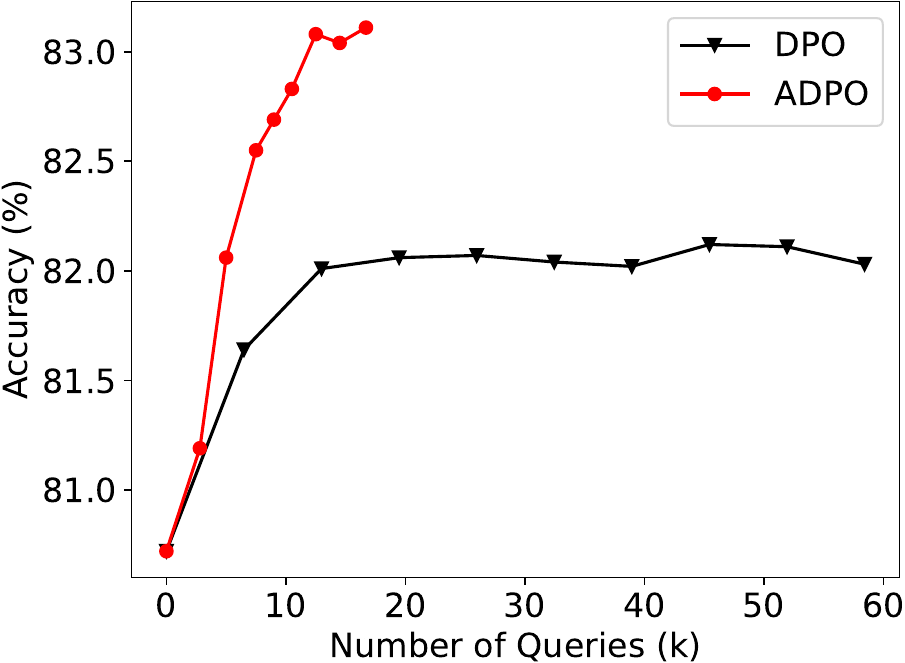}}
\caption{The test accuracy curve of DPO and $\method$ starting from Zephyr-Beta-SFT. The x-axis is the number of queries and the y-axis is the metric for corresponding dataset. Compared to DPO, $\method$ enjoys a faster performance improvement and a higher performance upper bound.}
\label{fig:curve_full}
\end{figure*}

\begin{figure*}[ht]
\centering   
\subfigure[ARC Challenge]{\label{fig:main_arc_gemma}\includegraphics[width=0.318\textwidth]{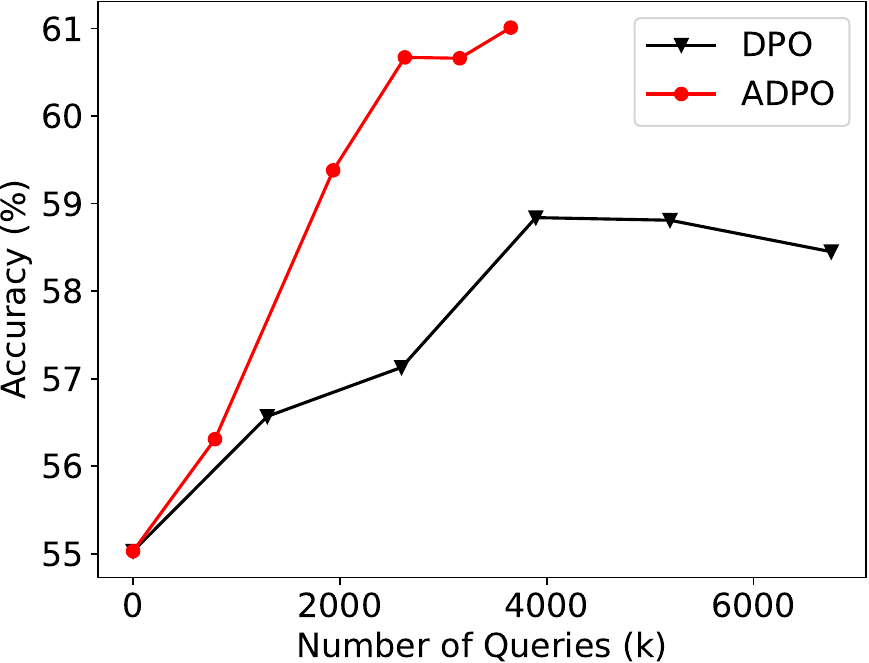}}
\subfigure[TruthfulQA]{\label{fig:main_truth_gemma}\includegraphics[width=0.318\textwidth]{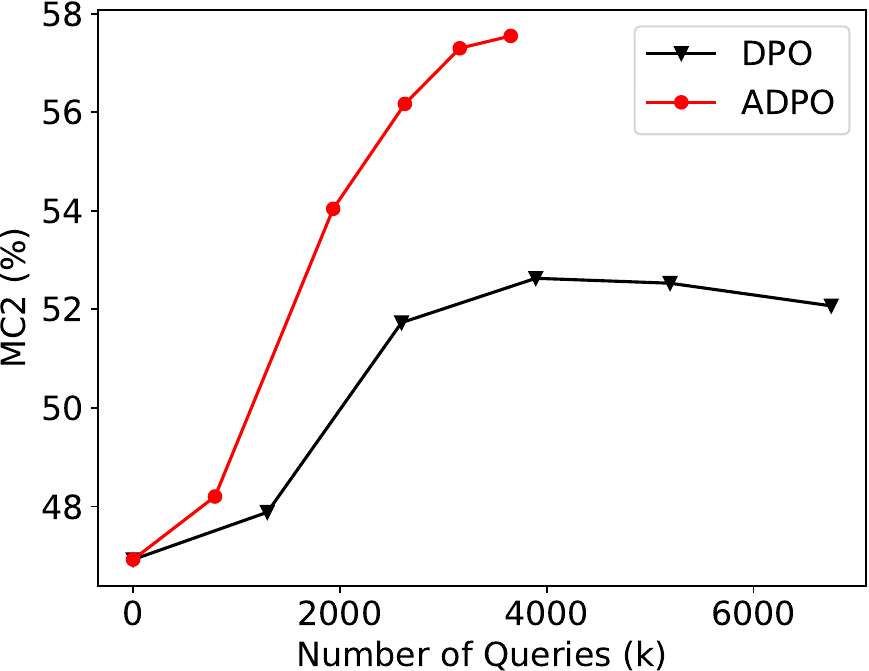}}
\subfigure[HellaSwag]{\label{fig:main_hella_gemma}\includegraphics[width=0.33\textwidth]{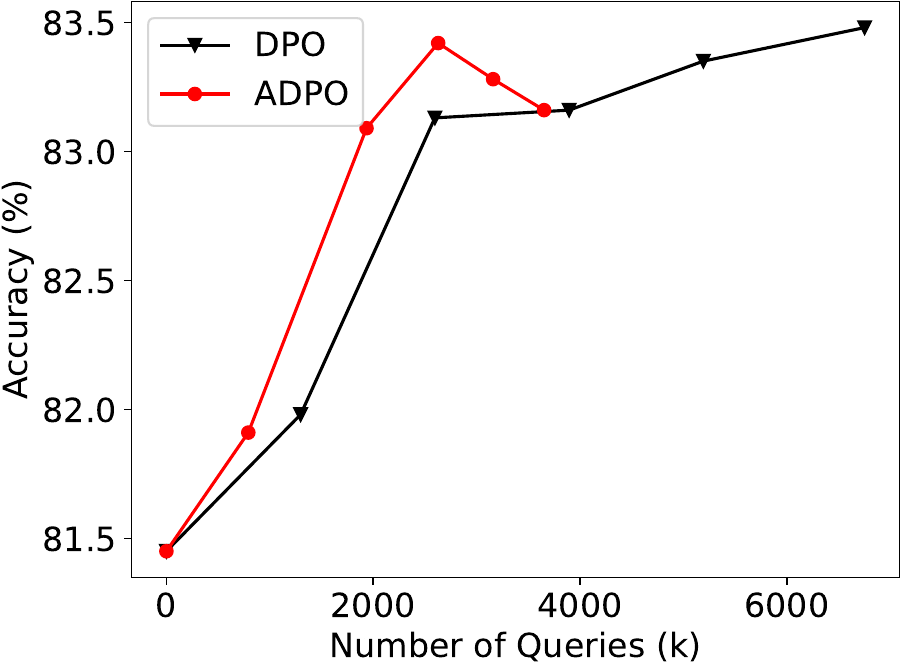}}
\caption{The test accuracy curve of DPO and $\method$ starting from Zephyr-Gemma-SFT. The x-axis is the number of queries and the y-axis is the metric for corresponding dataset. Compared to DPO, $\method$ enjoys a faster performance improvement and a higher performance upper bound.}
\label{fig:curve_gemma}
\end{figure*}

\subsection{Experimental Results}

\paragraph{Results on Objective Benchmarks }
The results on subjective benchmarks are presented in Table~\ref{tab:main-result}. We see that DPO and $\method$ improve the average score by a large margin starting from both Zephyr-Beta-SFT and Zephyr-Gemma-SFT. As for Zephyr-Beta, on TruthfulQA, ADPO outperforms DPO prominently by a margin of 7.1\%, and also outperforms DPO on ARC and HellaSwag by 1.08\% and 1.03\% respectively. Reflecting on the average score, we see that ADPO outperforms DPO by a margin of 3.08\%. As for Zephyr-Gemma, ADPO outperforms DPO prominently by a margin of~5.48\% on TruthfulQA and 2.56\% on ARC. ADPO also reaches a performance comparable to DPO on HellaSwag. Finally, reflecting on the average score, we see that ADPO outperforms DPO by a margin of 2.57\%. To sum up, results on both models shows the superiority of ADPO. Besides the performance on the benchmarks, we see that $\method$ only requires 16k queries for Zephyr-Beta and~3.6k for Zephyr-Gemma, which is only about a half of the size of the training dataset.

\paragraph{Results on Subjective Benchmarks } The results on subjective benchmarks are presented in Table~\ref{tab:main-result-alpaca-mt}. For Zephyr-Beta, we see that $\method$ achieves comparable performance with DPO. In detail, On MT-Bench, we see that $\method$ improves the average performance from $6.39$ to $7.20$, which is much more significant comparing to its gap with DPO of $0.21$. Similarly, on AlpacaEval, $\method$ also improve the LC win rate by ar margin of $8.08$, which is much more significant than its gap to DPO. 
For Zephyr-Gemma, we see that $\method$ outperforms DPO by a considerable margin. In detail, On MT-Bench, we see that $\method$ achieves a performance of $6.51$ compared $5.72$ achieved by DPO. Similarly, on AlpacaEval, $\method$ achieves a LC win rate of $15.85$, which also surpasses DPO by a large margin.

\paragraph{Query Efficiency } To further demonstrate the query efficiency for $\method$, we plot the test accuracy curves for $\method$ and DPO as the numbers of queries increase on selected datasets. The curves for ARC, HellaSwag and TruthfulQA starting from Zephyr-Beta-SFT and Zephyr-Gemma-SFT are shown in Figure~\ref{fig:curve_full} and Figure~\ref{fig:curve_gemma} respectively. 
For Zephry-Beta, we see that the growth of DPO's performance for ARC almost stops when query number reaches about 30k. This trend can also be observed on TruthfulQA after 20k queries and HellaSwag after about 15k queries. In contrast, the performance of $\method$ enjoys a faster improvement when training with the first about 10k results and maintains at a preferable level after that. 
For Zephyr-Gemma, we observe a similar pattern. The growing speed of the performance of DPO either slows down significantly after making 3k to 4k queries, as shown by Figure~\ref{fig:main_arc_gemma} and Figure~\ref{fig:main_truth_gemma}, or maintains at a very low level (Figure~\ref{fig:main_hella_gemma}).
These results suggest that $\method$ can effectively select the most informative data and only make queries for these preference labels.


\section{Conclusion and Future Work}

In this work, we considered query-efficient methods for aligning LLMs with human preference. We first formulated the problem as a contextual dueling bandit. Under linear reward and sub-optimal gap assumption, we proposed an active-learning-based algorithm, $\alg$. Our theoretical analysis shows that our algorithm enjoys a constant instance-dependent regret upper bound and query complexity. We then adapted our algorithm to direct preference optimization and proposed a query efficient DPO method, $\method$. We conducted experiments starting from two models, Zephyr-Beta-SFT and Zephyr-Gemma-SFT and evaluated the resulting models on both objective benchmarks and subjective benchmarks. Results show that, $\method$ achieves a comparable or even better performance compared to DPO with only less than half the demands on the human preference labels. Despite the good performance $\method$ achieves, since it uses DPO as the framework of our practical method, our theoretical analysis of $\alg$ cannot be directly applied to $\method$. We leave the theoretical analysis of $\method$ as our future work.


\appendix



\section{Further Discussions on Related Work}\label{app:discussion}
In this section, we provide further discussion on previous works \citep{wang2023rlhf,zhan2023query,wu2023making,sekhari2024contextual}, which consider query complexity in the dueling bandit setting, and explain why they fail to achieve an instance-dependent regret guarantee.

\subsection{Comparison with \citet{wang2023rlhf}}
\citet{wang2023rlhf} proposed a general framework, P2R, for efficiently querying human preferences, and later extended it to a white-box algorithm (P-OMLE) with a specialized analysis. However, the P2R algorithm relies on a comparison oracle that is stronger than ours. In the bandit setting, the oracle in \citet{wang2023rlhf} can return preference labels between responses to different prompts, which often exceeds the abilities of typical users. In contrast, our oracle only requires preferences between responses generated from the same prompt. Furthermore, P2R algorithms necessitate multiple independent comparisons between a baseline trajectory and user-generated trajectories, making it impractical to ask a single user for multiple independent preferences on the same query. Our ADPO algorithm, by contrast, only requires one preference feedback per query, making it much more user-friendly.

Additionally, in the linear reward setting, the query complexity for the Preference-based OMLE (white-box) algorithm is $\tilde O(d^2/\Delta^2)$\footnote{The query complexity of P-OMLE has a logarithmic dependency on the reward function space $\RR$. For the linear reward function class, where $\log \RR = d$, their complexity becomes $d \log \RR/\Delta^2 = d^2/\Delta^2$.}, which is the same as ours. However, P-OMLE requires solving an optimization problem over complex confidence regions, resulting in an intractable planning phase. In comparison, our APPO algorithm introduces an explicit confidence bonus to bypass this complexity and uses a policy optimization method, which is more tractable and closely aligned with practical RLHF methods, while still achieving the same query complexity.

\subsection{Failures in Achieving Instance-Dependent Regret Guarantees} 
Recently, \citet{zhan2023query} proposed a pure-exploration style algorithm (REGIME) that can identify the $\epsilon$-optimal policy with a query complexity of $\tilde{O}(d^2/\epsilon^2)$, where $d$ is the dimension of the feature space. However, it is important to note that the output policy may be a randomized policy, and does not guarantee a constant regret, even under the assumption of a minimal sub-optimality gap $\Delta$. Specifically, the minimal sub-optimality gap does not prevent a randomized policy from incurring regret between 0 and $\Delta$. Thus, even for an $\epsilon$ much smaller than the sub-optimality gap $\Delta$, there is no guarantee that a randomized algorithm will always achieve zero regret. For example, a policy that selects the optimal action with a probability of 50\% and a $\Delta$-suboptimal action with a probability of 50\% will result in a regret of $\epsilon = \Delta/2$. In this situation, the REGIME algorithm may lead to linear regret with respect to $T$ and fail to achieve instance-dependent regret guarantees.

A similar issue arises when transferring the sample complexity guarantee in \citet{wang2023rlhf} to an instance-dependent regret bound. Additionally, we observe that in Proposition 5 of \citet{wang2023rlhf}, the P2R framework achieves finite sample complexity using the UCBVI algorithm \citep{azar2017minimax}. It is important to note that the original UCBVI algorithm employs deterministic policies for each episode. However, \citet{azar2017minimax} only provides a $\sqrt{T}$ regret guarantee for the first $T$ rounds, rather than a sample complexity guarantee. To address this, \citet{jin2018q} demonstrates that any algorithm with sublinear regret can derive a finite sample complexity by randomly selecting a policy from the first $T$ rounds, resulting in a final randomized policy. After this step, the output policy may no longer be deterministic and fails to provide instance-dependent regret guarantees.
 
Another related work, \citet{wu2023making}, proposed a sampling-based algorithm (PR-LSVI) that provides a $\tilde{O}(d^3\sqrt{T})$ regret guarantee for the first $T$ rounds, which is not directly related to the sub-optimality gap. Consequently, a random mixture over the first $T$ rounds is required to identify a near-optimal policy, and it fails to achieve constant regret even in the presence of a positive sub-optimality gap.

As demonstrated above, all of these works can only find a random policy that achieves $\epsilon$-optimality, which cannot provide an instance-dependent regret guarantee, even with the assumption of a minimal sub-optimality gap.

\section{Additional Experiment Details}\label{app:add-exp-setup}

\paragraph{Hyper-parameters for Training Zephyr-Beta } We trained our models on 4$\times$NVIDIA A100 GPUs, with about 80G memory for each GPU. 
We set the learning rate to 5e-7 for both DPO and $\method$.
We use a linear learning rate scheduler with a warm-up ratio of $0.1$. The batch size per device is set to 4 and the gradients are accumulated every 4 steps, resulting in equivalent batch size 64. We set dropout to $0.1$ and the regularization parameter $\beta=0.1$ for both DPO and $\method$. For $\method$, we set the uncertainty threshold $\gamma=1.3$. We trained one epoch for both DPO and $\method$, which takes roughly 9 hours.

\paragraph{Hyper-parameters for Training Zephyr-Gemma } We trained our models on 4$\times$NVIDIA A100 GPUs, with about 80G memory for each GPU. 
We set the learning rate to 5e-7 for both DPO and $\method$.
We use a linear learning rate scheduler with a warm-up ratio of $0.1$. The batch size per device is set to 4 and the gradients are accumulated every 4 steps, resulting in equivalent batch size 64. We set dropout to $0.1$ and the regularization parameter $\beta=0.1$ for both DPO and $\method$. For $\method$, we set the uncertainty threshold $\gamma=1.5$. We trained one epoch for both DPO and $\method$, which takes roughly 1 hour. 

\paragraph{Evaluation Setup } For subjective evaluation benchmarks, we follow the standard setup specified in the original repositories. For obejctive benchmarks, we use few-show learning to prompt the LLMs. Specifically, the few-shot number of ARC is set to 25, HellaSwag to 10 and TruthfulQA to 0. We use \texttt{acc\_norm} as the metric for ARC and HellaSwag, and \texttt{mc2} for TruthfulQA. 



\section{Additional Experiment Results}

In this section, we present the results starting from Zephyr-Beta-SFT using LoRA-finetuning~\citep{hu2021lora}, which serves as an addition results.

\paragraph{Experiment Setup } We trained our models on 4$\times$NVIDIA RTX A6000 GPUs, with about 49G memory for each GPU. We set the LoRA rank to 64, $\alpha=16$, and dropout to $0.1$ and learning rate to 1e-5. We use a linear learning rate scheduler with a warm-up ratio of $0.1$. The batch size per device is set to 4 and the gradients are accumulated every 4 steps, resulting in equivalent batch size 64. We set the regularization parameter $\beta=0.1$ for both DPO and $\method$. For $\method$, we set the uncertainty threshold $\gamma=1.5$. We trained one epoch for both DPO and $\method$, which takes roughly 7 hours. We only evaluate the obtained checkpoints on objective benchmarks due to the costly nature of calling external large language models as the judge.

\begin{table*}[ht]
    \centering
    \caption{Result on objective benchmarks for LoRA finetuning on Zephyr-7B-Beta. $\method$ significantly outperforms DPO on ARC, TruthfulQA and HellaSwag. Besides, ADPO only makes 32k queries, which is about only a quarter to half of the queries made by DPO. }
    \begin{tabular}{l | c c c c | c}
    \toprule
        Models & ARC & TruthfulQA & HellaSwag & Average & \# Queries \\
        \midrule
Zephyr-Beta-SFT      & 58.28  & 40.36  & 80.72 & 59.79  & 0 \\
\midrule
Zephyr-Beta-DPO (LoRA) & 60.58 & 41.88 & 82.34 & 61.60 & 62k     \\
\rowcolor{LightCyan} 
Zephyr-Beta-$\method$ (LoRA) & \textbf{61.26} & \textbf{45.52} & \textbf{83.21} & \textbf{63.33} & \textbf{32k} \\
\bottomrule
\end{tabular}%
\label{tab:lora-result}
\end{table*}

\begin{figure*}[ht!]
\centering   
\subfigure[ARC Challenge]{\label{fig:main_arc}\includegraphics[width=0.33\textwidth]{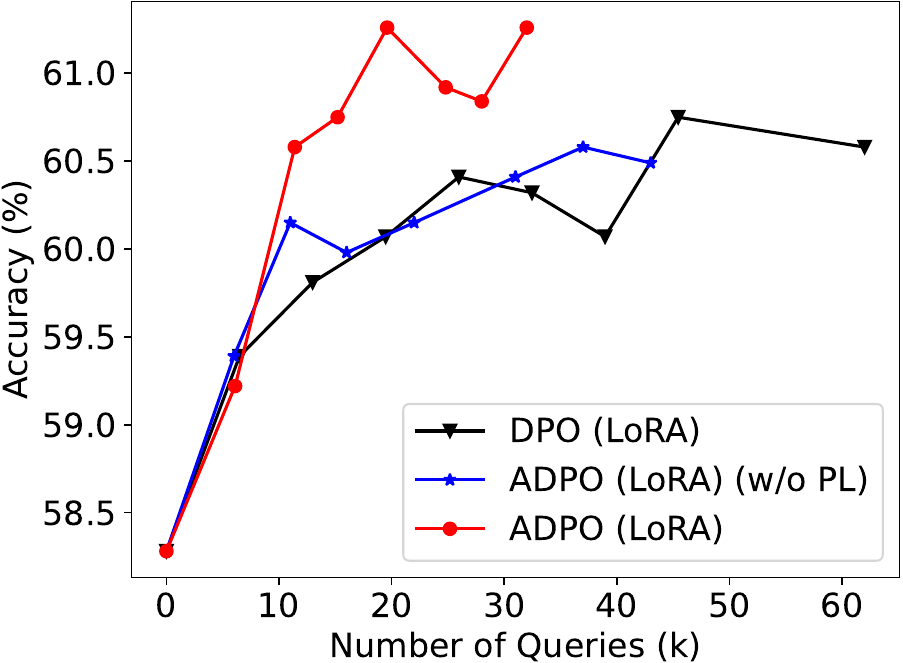}}
\subfigure[TruthfulQA]{\label{fig:main_truth}\includegraphics[width=0.318\textwidth]{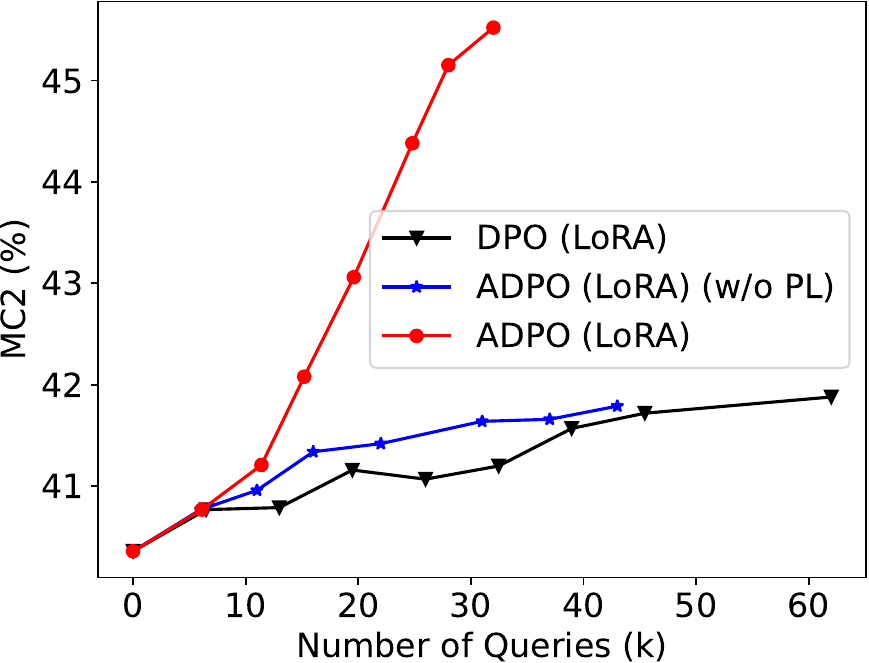}}
\subfigure[HellaSwag]{\label{fig:main_hella}\includegraphics[width=0.33\textwidth]{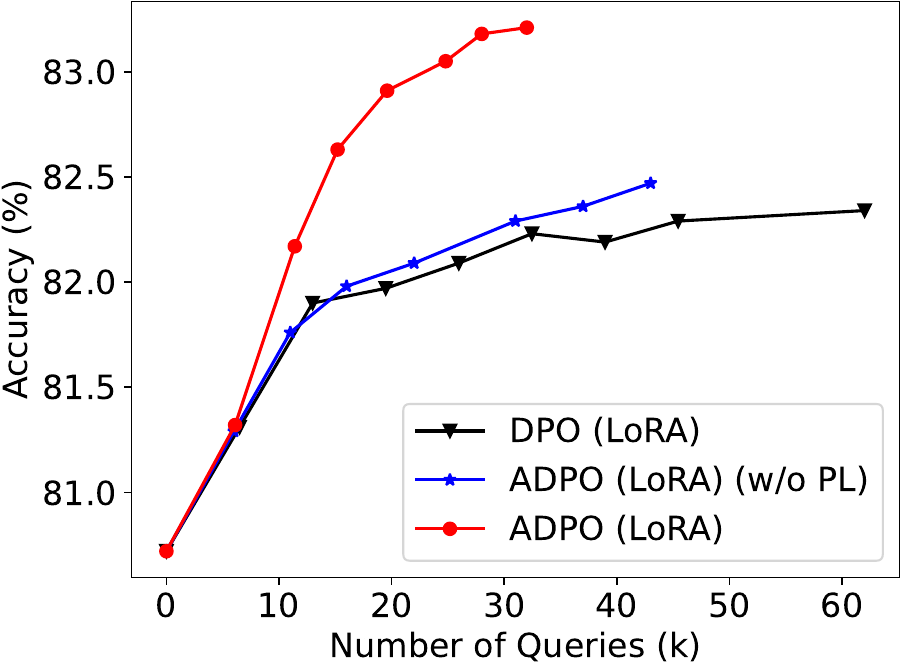}}
\caption{The test accuracy curve of DPO, $\method$ (w/o PL) and $\method$ under LoRA-finetune. The x-axis is the number of queries and the y-axis is the metric for corresponding dataset. Compared to DPO and $\method$ (w/o PL), $\method$ enjoys a faster performance improvement and a higher performance upper bound.}
\label{fig:curve}
\end{figure*}

\paragraph{Benchmark Performances } The results are summarized in Table~\ref{tab:lora-result}. We observe a similar results as for full-finetuning. The results show that both DPO and $\method$ improve the average score by a large margin. ADPO outperforms DPO on TruthfulQA by a relatively large margin of 3.64\% and also reaches an at-least comparable performance on other three datasets. Finally, reflecting on the average score, we see that ADPO outperforms DPO by a margin of 1.73\%. Besides, we see that $\method$ only requires 32k queries, which is only about half of the size of the training dataset. These results show that with much less number of queries, $\method$ can reach a comparable or even superior performance than DPO, which is consistent with results under full-finetune.

\paragraph{Query Efficiency } We plot a set of similar test accuracy curves for $\method$ and DPO for LoRA-finetuning. The results are presented in Figure~\ref{fig:curve}. Here we observe a similar pattern. The growing speed of the performance of DPO either slows down significantly after making 15k to 30k queries, as shown by Figure~\ref{fig:main_arc} and Figure~\ref{fig:main_truth}, or maintains at a very low level (Figure~\ref{fig:main_hella}).
These results suggest that $\method$ can effectively select the most informative data and only make queries for these preference labels. 


\section{Ablation Studies}

In this section, we study the impact of involving pseudo labels in the training process. Due to the time-costly nature of training language models, we consider LoRA in our ablation study. 

\subsection{Impact of Pseudo Labels} 

\begin{table*}[ht!]
    \centering
    \caption{The effect of pseudo-labels under LoRA-finetune setting. $\method$ performs better than $\method$ (w/o PL) in terms of average scores with fewer queries.}
    \begin{tabular}{l c | c c c c | c}
    \toprule
        Model & $\gamma$ & ARC & TruthfulQA & HellaSwag & Average & \# Queries \\
        \midrule
$\method$ (LoRA) (w/o PL) & 0.8 & 60.49 & 41.39  & 82.23 & 61.37 & 38k \\
$\method$ (LoRA) (w/o PL) & 1.0 & 60.49 & 41.62  & 82.32 & 61.48 & 40k \\
$\method$ (LoRA) (w/o PL) & 1.2 & 60.49 & 41.79 & 82.47 & 61.58 & 43k \\
\midrule
$\method$ (LoRA) & 1.5 & \textbf{61.26} & \textbf{45.52} & \textbf{83.21} & \textbf{63.33} & \textbf{32k} \\
    \bottomrule
    \end{tabular}%
    \label{tab:neglect}
\end{table*}

An alternative to active learning is to directly follow Algorithm~\ref{algo:po} and simply neglect those training data with high confidence. 
Since neglected samples will not affect the loss and the corresponding gradient, we set the label to $0$ so that they will not contribute to $\nabla_{\boldsymbol{\theta}}\mathcal{L}$ in Eq.~\eqref{eq:emloss} during the learning process. Formally, we define the label $o_{\boldsymbol{\theta}}'$ as follows:
\begin{align*}
    o_{\btheta}'(x, y^1, y^2) = \begin{cases}
      o & \text{if $C_{\btheta}(x, y^1, y^2) \leq \gamma$}\\
      0 & \text{if $C_{\btheta}(x, y^1, y^2) > \gamma$}
    \end{cases}.
\end{align*}
We keep the remaining part of our method the same and denote this method as ``$\method$ (LoRA) (w/o PL)''. In this experiment, we also conduct a grid search for the confidence threshold $\gamma$ and finally pick $\gamma$ to be 0.8, 1.0, and 1.2. The performances of the trained models are shown in Table~\ref{tab:neglect}. We also plot the training curve in Figure~\ref{fig:curve}. The results show that, without pseudo-labels, the performance suffers from a significant downgrade in average score compared to $\method$. The training curves further indicate that, without pseudo labels, the training dynamics are much more similar to vanilla DPO. These results show the crucial role of pseudo-labels in the active learning process.

\subsection{Value of Confidence Threshold}\label{app:ablation-thres}

\begin{table*}[ht]
    \centering
    \caption{The effect of confidence threshold in practice setting. We vary the value of $\gamma$ and report the evaluation results. When $\gamma$ is increasing, $\method$ made more queries and the performance pattern is getting closer to DPO.}
    \begin{tabular}{c c | c c c c | c}
    \toprule
    Method &$\gamma$ & ARC & TruthfulQA  & HellaSwag & Average & \# Queries \\
        \midrule
DPO  & - & 60.58 & 41.88 & 82.34 & 61.60     & 35.7k \\
\midrule
\multirow{4}{*}{$\method$} & 1.0 & 61.01 & 48.41 & 83.35 & 64.26     & 16k   \\
& 1.3 & 61.43 & 47.56 & 83.48 & 64.16 & 24k   \\
& 1.5 & 61.26 & 45.52 & 83.21 & 63.33  & 32k   \\
& 1.8 & 60.92 & 43.20 & 82.13 & 62.08 & 43k \\  
    \bottomrule
    \end{tabular}%
    \label{tab:uncertainty}
\end{table*}
\color{black}

We study the impact of different confidence thresholds. We varies the value of $\gamma$ to 1.0, 1.3, 1.5, and 1.8. For each $\gamma$, we count the preference labels used by the models and evaluate the trained models on the datasets. As shown in Table~\ref{tab:uncertainty}, when the confidence threshold is small, with more predicted labels, these models perform better on the TruthfulQA dataset. On the other hand, when the confidence threshold goes larger, the models are making more queries, and the performance patterns become closer to the DPO baseline. 
Another observation is that for all our chosen $\gamma$, $\method$ consistently outperforms the DPO baseline, which implies that $\method$ is not very sensitive to the uncertainty threshold and a coarse grid search of confidence threshold can introduce a fairly good performance.

\section{Proof of Theorems in Section \ref{sec:main-result}}\label{app:proof-of-main}

In this section, we provide the proof of Theorem \ref{thm:main} and we first introduce several lemmas. The following lemma provides an upper bound on the query complexity and the corresponding dataset size $|\cC_T|$.
\begin{lemma}[Modified from Lemma 4.5, \citealp{zhang2023interplay}]\label{lem:label-cplx}
Given a uncertainty threshold $0 < \Gamma \leq 1$, if we set the regularization parameter $\lambda = B^{-2}$, then for each round $t\in [T]$, we have $|\cC_t| \leq |\cC_T| \leq 16d\Gamma^{-2}\log(3LB\Gamma^{-1})$.
\end{lemma}
For a finite dataset $\cC_T$, the following lemma provides a upper bound for the estimation error between $\hat{\btheta}_t$ and $\btheta^*$.
\begin{lemma}\label{lem:ucb}
Suppose we have $\|\btheta^*\| \leq B$, $\|\bphi(x, y)\| \leq L/2$. Then with probability at least $1-\delta$, for each round $t\in [T]$, we have
\begin{align*}
    \|\btheta^* - \hat{\btheta}_t \|_{\Sigma_{t-1}} \leq \frac{1}{\kappa_\mu}\cdot \big(\sqrt{\lambda}B + \sqrt{2d\log (\lambda + |\cC_{T}|L^2/d\lambda \delta)}\big),
\end{align*}
\end{lemma}
Based on Lemmas \ref{lem:label-cplx} and \ref{lem:ucb}, the following  auxiliary lemma proposes a proper choice for the uncertainty threshold $\Gamma$ and confidence radius $\beta$ in Algorithm \ref{algo:po}.
\begin{lemma}\label{lem:beta-gamma}
If we set the  uncertainty threshold $\Gamma = \kappa_{\mu}\Delta/(2d\iota_1)$ and confidence radius $\beta = \kappa_{\mu}^{-1}( 1 + 4\sqrt{d\iota_2} + \sqrt{2d\iota_3})$, where $\iota_1 = 42\log(126LB\sqrt{d}\Delta^{-1}\kappa_{\mu}^{-1}) + \sqrt{8\log(1/\delta)}$, $\iota_2 = \log(3LB\Gamma^{-1})$ and $\iota_3 = \log\big((1 + 16L^2B^2\Gamma^{-2}\iota_2)/\delta\big)$, then we have $2\beta \Gamma < \Delta$ and
\begin{align*}
    \beta \ge \frac{1}{\kappa_\mu}\cdot \big(\sqrt{\lambda}B + \sqrt{2d\log (\lambda + |\cC_{T}|L^2/d\lambda \delta)}\big).
\end{align*}
\end{lemma}
With these parameters, we now define the event $\cE_1$ as 
 \begin{equation*}
    \event_1 = \{ \forall t \in [T],  \|\hat{\btheta}_t - \btheta^* \|_{\bSigma_{t-1}^{-1}} \leq \beta \}.
\end{equation*}
According to Lemma \ref{lem:ucb} and Lemma \ref{lem:beta-gamma}, we have $\Pr(\cE_1)\ge 1-\delta$. Conditioned on the event $\cE_1$, the following lemma suggests that our estimated discrepancy is no less than the actual discrepancy.
\begin{lemma}\label{optimistic}
    On the event $\cE_1$, for each round $t\in[T]$, context $x\in \cX$ and any action $y\in \cA$,
    the estimated discrepancy $\hat{D}_t(x,y)$ satisfied
    \begin{align*}
        \hat{D}_t(x,y)\ge D_t(x,y)=\langle \btheta^*, \bphi(x,y) - \bphi_t^2 \rangle.
    \end{align*}
    On the other hand, we have
    \begin{align*}
        \hat{D}_t(x,y)\leq D_t(x,y) + 2 \beta \|\bphi(x,y) - \bphi^2_t\|_{\bSigma_{t-1}^{-1}}.
    \end{align*}
\end{lemma}
It is worth to notice that in Algorithm \ref{algo:po} (Line \ref{line:oppo}), we update the policy $\pi_t$ with online mirror descent and the following lemma provides the regret guarantee for this process.
\begin{lemma}[Modified from Lemma 6.2, \citealp{he2022near}]\label{lem:omd}
For any estimated value function $\hat{D}_t(x, \cdot)$, if we update the policy $\pi_{t+1}(\cdot | x)$ by the exponential rule: 
\begin{equation}\label{eq:one-step-descent}
    \pi_{t+1}(\cdot | x) \propto \pi_{t}(\cdot | x)\cdot \exp\big(\eta \hat{D}_t(x, \cdot)\big),
\end{equation}
then the expected sub-optimality gap at round $T$ can be upper bounded as follows:
\begin{align*}
    & \EE_{x\sim \cD,y \sim \pi^*(\cdot|x)}[\hat{D}_t(x, y)] - \EE_{x\sim \cD,y \sim \pi_t(\cdot|x)}[\hat{D}_t(x, y)] \\
    & \quad \leq 2\eta + \eta^{-1}\EE_{x\sim \cD}\Big[ \text{KL}\big(\pi^*(\cdot | x)\|\pi_t(\cdot | x)\big) -  \text{KL}\big(\pi^*(\cdot | x)\|\pi_{t+1} (\cdot | x)\big) \Big] 
\end{align*}
\end{lemma}
With the help of these lemmas, we are now ready to prove our main theorem.
\begin{proof}[Proof of Theorem~\ref{thm:main}]
Now we start the regret analysis. For simplicity, for each round $t\in[T]$, we use $\bphi_t$ to denote $\bphi(x_t, y_t)$. Initially, the episodes and their corresponding regret can be decomposed into two groups based on whether episode $t$ is added to the dataset $\mathcal{C}_T$:
\begin{align}
    \text{Regret}(T) &=\sum_{t=1}^T \langle \btheta^*, \bphi_t^* \rangle - \langle \btheta^*, \bphi_t^1 \rangle\notag\\
    &= \sum_{t=1}^T D_t(x_t, y_t^*) - D_t(x_t, y^1_t)\notag\\
    &= \underbrace{\sum_{t \in \cC_T}  D_t(x_t, y_t^*) - D_t(x_t, y^1_t)}_{I_1} +  \underbrace{\sum_{t \notin \cC_T}  D_t(x_t, y_t^*) - D_t(x_t, y^1_t)}_{I_2} \label{eq:0001}
\end{align}
where $D_t(x, y) = \langle \btheta^*, \bphi(x,y)-\bphi_t^2 \rangle$ denotes the reward gap between action $y \in \cA$ and selected action $y_t^2$ at round $t$.

Now, we bound this two term separately. For the term $I_1$, we have
\begin{align}
   I_1 &= \sum_{t \in \cC_T}  D_t(x_t, y_t^*) - D_t(x_t, y^1_t)\notag\\
   &\leq  \underbrace{\sum_{t \in \cC_T} \hat{D}_t(x_t, y_t^1) - D_t(x_t, y_t^1)}_{J_1} + \sum_{t \in \cC_T} \hat{D}_t(x_t, y_t^*) - \hat{D}_t(x_t, y^1_t)\notag\\
    &= J_1+\underbrace{\sum_{t \in \cC_T}  \EE_{x_t\sim \cD,y \sim \pi^*(\cdot|x)}[\hat D_t(x_t, y)] - \EE_{x_t\sim \cD,y \sim \pi_t(\cdot|x)}[\hat D_t(x_t, y)] }_{J_2}\notag\\
    &\qquad +\underbrace{\sum_{t \in \cC_T} \hat{D}_t(x_t, y_t^*) - \hat D_t(x_t, y_t^1) - \sum_{t \in \cC_T}  \EE_{x_t\sim \cD,y \sim \pi^*(\cdot|x)}[\hat D_t(x_t, y)] - \EE_{x_t\sim \cD,y \sim \pi_t(\cdot|x)}[\hat D_t(x_t, y_t)]}_{J_3},\label{eq:0002}
\end{align}
where the inequality holds due to Lemma~\ref{optimistic}.

For the term $J_1$, we have
\begin{align}\label{eq:elliptical}
    J_1 & = \sum_{t \in \cC_T}  \hat D_t(x_t, y_t^1) - D_t(x_t, y^1_t)\notag\\
    &\leq \sum_{t \in \cC_T}  \min \{4, 2\beta\|\bphi_t^1-\bphi_t^2\|_{\bSigma_{t-1}^{-1}} \} \notag \\
    &\leq 4\beta \sqrt{|\cC_T|\cdot \sum_{t \in \cC_T}  \min \{1, \|\bphi_t^1-\bphi_t^2\|_{\bSigma_{t-1}^{-1}}^2 \}}\notag\\
    & \leq 8\beta \sqrt{|\cC_T| d\log \bigg(\frac{\lambda d + |\cC_T|L^2}{\lambda d}}\bigg),
\end{align}
where the first inequality holds due to Lemma \ref{optimistic} with the fact that $-2\leq D_t(x_y,y_t^1)\leq 2$, the second inequality holds due to Cauchy–Schwarz inequality and the last inequality holds due to the elliptical potential lemma (Lemma \ref{Lemma:abba}).

 The term $J_2$ reflects the sub-optimality from the online mirror descent process and can be upper bounded by Lemma \ref{lem:omd}. For simplicity, we denote $\cC_T = \{ t_1, ...,t_K \}$ where $K = |\cC_T|$. Thus, we have
\begin{align}\label{eq:omd}
    J_2 &= \sum_{k=1}^K  \EE_{x_{t_k}\sim \cD,y \sim \pi^*(\cdot|x)}[\hat D_{t_k}(x_{t_k}, y)] - \EE_{x_{t_k}\sim \cD,y \sim \pi_{t_k}(\cdot|x)}[\hat D_{t_k}(x_{t_k}, y)] \notag \\
    & \leq \sum_{k=1}^K \big(2\eta + \eta^{-1}\EE_{x\sim \cD}\big[ \text{KL}(\pi^*(\cdot | x)\|\pi_{t_k}(\cdot | x)) -  \text{KL}(\pi^*(\cdot | x)\|\pi_{t_{k}+1} (\cdot | x) \big] \big) \notag \\
    & = 2\eta K + \eta^{-1}\EE_{x\sim \cD}\big[ \text{KL}(\pi^*(\cdot | x)\|\pi_1(\cdot | x)) -  \text{KL}(\pi^*(\cdot | x)\|\pi_{t_{K}+1} (\cdot | x) \big] \notag \\
    &\leq 2\eta K + \eta^{-1}\EE_{x\sim \cD}\big[\text{KL}(\pi^*(\cdot | x)\|\pi_1(\cdot | x))\big] \notag\\
    & \leq 2\sqrt{32d\Gamma^{-2}\log(3LB\Gamma^{-1})\log|\cA|},
\end{align}
where the first inequality holds due to Lemma~\ref{lem:omd}, the second equation holds due to policy $\pi$ keeps unchanged for $t\in \cC_T$, the second inequality holds due to $\text{KL}(\cdot\|\cdot)\ge 0$
and the last inequality holds due to $\eta = \sqrt{\Gamma^2\log \cA / \big(32d\log(3LB\Gamma^{-1})\big)}$ with the fact that $\pi_1$ is uniform policy.

According to Azuma-Hoeffding's inequality (Lemma \ref{lemma:azuma}), with probability at least $1-\delta$, the term $J_3$ can be upper bounded by
\begin{align}
    J_3 \leq 2\sqrt{2|\cC_T| \log (1/\delta)}.\label{eq:0003}
\end{align}
Substituting \eqref{eq:elliptical}, \eqref{eq:omd} and \eqref{eq:0003} into \eqref{eq:0002}, we have
\begin{align}
    I_1& =J_1+J_2+J_3 \\
    &\leq 8\beta \sqrt{|\cC_T| d\log \bigg(\frac{\lambda d + |\cC_T|L^2}{\lambda d}}\bigg)+ 2\sqrt{32d\Gamma^{-2}\log(3LB\Gamma^{-1})\log|\cA|} +  2\sqrt{2|\cC_T| \log (1/\delta)}\notag\\
    &\leq \tilde O(\beta d/\Gamma)\notag\\
    &=\tilde O\bigg(\frac{d^2}{\Delta}\bigg).\label{eq:0004}
\end{align}
where the last inequality holds due to Lemma \ref{lem:label-cplx}.

Now, we only need to focus on the term $I_2$. For each round $t \notin \cC_T$, we have
\begin{align*}
    D_t(x_t, y_t^*) - D_t(x_t, y_t^1) & = \langle \btheta^* - \hat{\btheta}_t, \bphi_t^* - \bphi_t^2 \rangle + \langle \hat{\btheta}_t, \bphi_t^* - \bphi_t^2 \rangle -  \langle \btheta^*, \bphi_t^1 - \bphi_t^2 \rangle \\
    & \leq \beta\|\bphi_t^*-\bphi_t^2\|_{\bSigma_{t-1}^{-1}} +  \langle \hat{\btheta}_t, \bphi_t^* - \bphi_t^2 \rangle -  \langle \btheta^*, \bphi_t^1 - \bphi_t^2 \rangle \\
    & \leq \beta\|\bphi_t^1-\bphi_t^2\|_{\bSigma_{t-1}^{-1}} +  \langle \hat{\btheta}_t, \bphi_t^1 - \bphi_t^2 \rangle -  \langle \btheta^*, \bphi_t^1 - \bphi_t^2 \rangle \\
    & \leq 2\beta\|\bphi_t^1-\bphi_t^2\|_{\bSigma_{t-1}^{-1}},
\end{align*}
where the first inequality holds due to Lemma \ref{optimistic}, the second inequality holds due to the selection rule of action $\bphi_t^1$ and the last inequality holds due to Lemma \ref{optimistic}. According to the definition of set $\cC_T$ in Algorithm \ref{algo:po}, for each round $t \notin \cC_T$, we have $\|\bphi_t^1-\bphi_t\|_{\bSigma_{t-1}^{-1}} \leq \Gamma$. Therefore, the sub-optimality gap at round $t$ is upper bounded by
\begin{equation*}
    2\beta\|\bphi_t^1-\bphi_t\|_{\bSigma_{t-1}^{-1}} \leq 2 \beta \Gamma < \Delta,
\end{equation*}
where the second inequality holds due to Lemma \ref{lem:beta-gamma}. According to the minimal sub-optimality assumption (Assumption \ref{assumption:gap}), this indicates that the regret yielded in round $t\notin \cC_T$ is 0. Summing up over $t \notin \cC_T$, we have
\begin{align}
    I_2=\sum_{t\in \cT_t} D_t(x_t, y_t^*) - D_t(x_t, y_t^1) =0.\label{eq:0005}
\end{align}
Combining the results in \eqref{eq:0004} and \eqref{eq:0005}, we complete the proof of Theorem \ref{thm:main}.
\end{proof}

\section{Proof of Lemmas in Appendix \ref{app:proof-of-main}}

In this section, we provide the proofs of the lemmas in Appendix~\ref{app:proof-of-main}.

\subsection{Proof of Lemma~\ref{lem:label-cplx}}

\begin{proof}[Proof of Lemma \ref{lem:label-cplx}]
The proof follows the proof in \citet{zhang2023interplay}. Here we fix the round $t$ to be $T$ in the proof and only provide the upper bound of $\cC_T$ due to the fact that $\cC_t$ is monotonically increasing w.r.t. the round $t$. For all selected episode $t \in \cC_T$, since we have $\|\bphi_t^1 - \bphi_t^2\|_{\bSigma_{t-1}^{-1}} \ge \Gamma$, the summation of the bonuses over all the selected episode $t \in \cC_T$ is lower bounded by
\begin{align}
    \sum_{t \in \cC_T} \min\left\{1, \|\bphi_t^1 - \bphi_t^2\|_{\bSigma_{t-1}^{-1}}^2\right\} \ge |\cC_T|\min\{1, \Gamma^2\} = |\cC_T|\Gamma^2,\label{eq:lowerbounds}
\end{align}
where the last equation holds due to $0 \leq \Gamma \leq 1$. On the other hand, according to Lemma~\ref{lm:elliptical}, the summation is upper bounded by:
\begin{align}
    \sum_{t \in \cC_T} \min\left\{1, \|\bphi_t^1 - \bphi_t^2\|_{\bSigma_{t-1}^{-1}}^2\right\} \le 2d \log\left(\frac{\lambda d + |\cC_T|L^2}{\lambda d}\right)\label{eq:upperbounds}.
\end{align}
Combining~\eqref{eq:lowerbounds} and~\eqref{eq:upperbounds}, we know that the total number of the selected data points $|\cC_T|$ satisfies the following inequality:
\begin{align*}
    \Gamma^2|\cC_T| \le 2d \log\left(\frac{\lambda d + |\cC_T|L^2}{\lambda d}\right).
\end{align*}
For simplicity, we reorganized the result as follows:
\begin{align}
    \frac{\Gamma^2|\cC_T|}{2d} \le \log\left(1 + \frac{2L^2}{\Gamma^2\lambda}\frac{\Gamma^2|\cC_T|}{2d}\right).\label{eq:thres}
\end{align}
Notice that $\lambda = B^{-2}$ and $2L^2B^2 \ge 2 \ge \Gamma^2$, therefore, if $|\cC_T|$ is too large such that 
\begin{align*}
    \frac{\Gamma^2 |\cC_T|}{2d} > 4\log\left(\frac{4L^2B^2}{\Gamma^2}\right) + 1 \ge 4\log\left(\frac{4L^2B^2}{\Gamma^2}\right) + \frac{\Gamma^2}{2L^2B^2},
\end{align*}
then according to Lemma~\ref{lem:log2lin}, \eqref{eq:thres} will not hold. Thus the necessary condition for \eqref{eq:thres} to hold is:
\begin{align*}
    \frac{\Gamma^2 |\cC_T|}{2d} \le 4\log\left(\frac{4L^2B^2}{\Gamma^2}\right) + 1 = 8\log\left(\frac{2LB}{\Gamma}\right) + \log(e) =  8\log\left(\frac{2LBe^{\frac18}}{\Gamma}\right) < 8\log\left(\frac{3LB}{\Gamma}\right).
\end{align*}
Applying basic calculus, we obtain the claimed bound for $|\cC_T|$ and thus complete the proof of Lemma~\ref{lem:label-cplx}.
\end{proof}

\subsection{Proof of Lemma \ref{lem:ucb}}

\begin{proof}[Proof of Lemma \ref{lem:ucb}]
This proof follows the proof in \citet{di2023variance}. For each round $t \in [T]$, we define the following auxiliary quantities:
\begin{align*}
&G_t(\btheta) = \lambda \btheta + \sum_{\tau \in \cC_{t-1}} \Big[\mu\big((\bphi_\tau^1-\bphi_\tau^2)^\top \btheta\big) - \mu\big((\bphi_\tau^1-\bphi_\tau^2)^\top \btheta^*\big)\Big](\bphi_\tau^1-\bphi_\tau^2)\\
& \epsilon_{t} = o_t - \mu\big((\bphi_t^1-\bphi_t^2)^\top \btheta^*\big)\\
& Z_t = \sum_{\tau \in \cC_{t-1}} \epsilon_\tau (\bphi_\tau^1-\bphi_\tau^2).
\end{align*}
By defining $\hat{\btheta}_t$ as the solution to~\eqref{eq:mle}, we plug the equation into the definition of $G_t$ and we have
\begin{align*}
    G_t (\hat{\btheta}_t) &= \lambda \hat{\btheta}_t  + \sum_{\tau \in \cC_{t-1}} \Big[\mu\big((\bphi_\tau^1-\bphi_\tau^2)^\top \hat{\btheta}_t) - o_\tau + o_\tau - \mu\big((\bphi_\tau^1-\bphi_\tau^2)^\top \btheta^*\big)\Big](\bphi_\tau^1-\bphi_\tau^2)\\
    & = \lambda \hat{\btheta}_t  + \sum_{\tau \in \cC_{t-1}} \Big[\mu\big((\bphi_\tau^1-\bphi_\tau^2)^\top \hat{\btheta}_t) - o_\tau \Big](\bphi_\tau^1-\bphi_\tau^2) \\
    & \quad + \sum_{\tau \in \cC_{t-1}} \Big[ o_\tau - \mu\big((\bphi_\tau^1-\bphi_\tau^2)^\top \btheta^*\big)\Big](\bphi_\tau^1-\bphi_\tau^2)\\
    & = Z_t.
\end{align*}
Therefore, we have that
\begin{equation*}
    G_t(\hat{\btheta}_t) - G_t(\btheta^*) =Z_t- G_t(\btheta^*) = Z_t - \lambda \btheta^*.
\end{equation*}
On the other hand, applying Taylor's expansion, there exists $\alpha\in[0,1]$ and $\tilde \btheta_t= \alpha \hat{\btheta}_t +(1-\alpha )\btheta^*$, such that the following equation holds:
\begin{align*}
    G_t(\hat{\btheta}_t) - G_t(\btheta^*) &= \lambda (\hat{\btheta}_t - \btheta^*) + \sum_{\tau \in \cC_{t-1}} \Big[\mu\big((\bphi_\tau^1-\bphi_\tau^2)^\top \btheta\big) - \mu\big((\bphi_\tau^1-\bphi_\tau^2)^\top \btheta^*\big)\Big](\bphi_\tau^1-\bphi_\tau^2) \\
    & = \Big[\lambda \Ib + \sum_{\tau \in \cC_{t-1}}  \mu'\big((\bphi_\tau^1-\bphi_\tau^2)^\top \tilde{\btheta}_t \big)(\bphi_\tau^1-\bphi_\tau^2)(\bphi_\tau^1-\bphi_\tau^2)^\top \Big] (\hat{\btheta}_t - \btheta^*) \\
    & = F(\tilde{\btheta}_t) (\hat{\btheta}_t - \btheta^*),
\end{align*}
where we define $
    F(\tilde{\btheta}_t)=\lambda \Ib + \sum_{\tau \in \cC_{t-1}}  \mu'\big((\bphi_\tau^1-\bphi_\tau^2)^\top \tilde{\btheta}_t \big)(\bphi_\tau^1-\bphi_\tau^2)(\bphi_\tau^1-\bphi_\tau^2)^\top$.
Thus, we have:
\begin{align*}
    \|\hat \btheta_{t} - \btheta^*\|^2_{\bSigma_{t-1}} &= (Z_t - \lambda \btheta^*)^\top F(\tilde{\btheta}_t)^{-1} \bSigma_{t-1} F(\tilde{\btheta}_t)^{-1}(Z_t - \lambda \btheta^*) \\
    & \leq \frac{1}{\kappa_{\mu}^2} (Z_t - \lambda \btheta^*)^\top \bSigma_{t-1}^{-1} (Z_t - \lambda \btheta^*) \\
    & = \frac{1}{\kappa_{\mu}^2} \| Z_t - \lambda \btheta^* \|_{\bSigma_{t-1}^{-1}}^2
\end{align*}
where the inequality holds due to $F(\tilde \btheta_t) \succeq \kappa_\mu\hat \bSigma_{t-1}$. Now we have:
\begin{align*}
    \|\hat \btheta_{t} - \btheta^*\|_{\bSigma_{t-1}} & \leq \frac{1}{\kappa_{\mu}} \| Z_t - \lambda \btheta^* \|_{\bSigma_{t-1}^{-1}} \\
    & \leq \frac{1}{\kappa_{\mu}} \big( \|Z_t\|_{\bSigma_{t-1}^{-1}} + \| \lambda \btheta^* \|_{\bSigma_{t-1}^{-1}} \big) \\
    & \leq \frac{1}{\kappa_{\mu}} \big( \|Z_t\|_{\bSigma_{t-1}^{-1}} + \sqrt{\lambda} B \big),
\end{align*}
where the second inequality holds due to triangle inequality and last inequality holds due to $\bSigma_{t-1} \succeq \lambda \Ib$. Now we only need to bound $\|Z_t\|_{\bSigma_{t-1}^{-1}}$. 

According to Lemma~\ref{lem:concentrate}, with probability at least $1-\delta$, we have 
\begin{align*}
    \|Z_t\|_{\bSigma_{t-1}^{-1}} \leq \sqrt{2 \log \bigg(\frac{\sqrt{\det(\bSigma_{t-1})}}{\sqrt{\det(\bSigma_0)}\delta}\bigg)}  \leq \sqrt{2 \log \bigg(\frac{\det(\bSigma_{t-1})}{\lambda^d \delta}\bigg)} \leq \sqrt{2 d \log \bigg(\frac{\lambda + |\cC_t|L^2/d}{\lambda^d \delta}\bigg)}
\end{align*}
where the first inequality holds due to Lemma~\ref{lem:concentrate} and the last inequality holds due to Lemma~\ref{lem:determine}. Now we combine the two term and have
\begin{equation*}
     \|\hat \btheta_{t} - \btheta^*\|_{\bSigma_{t-1}} \leq \frac{1}{\kappa_\mu} \Bigg(\sqrt{\lambda}B + \sqrt{2 d \log \bigg(\frac{\lambda + |\cC_t|L^2/d}{\lambda^d \delta}\bigg)}\Bigg),
\end{equation*}
which concludes our statement.
\end{proof}

\subsection{Proof of Lemma \ref{lem:beta-gamma}}
\begin{proof}[Proof of Lemma \ref{lem:beta-gamma}]
This proof follows the proof in \citet{zhang2023interplay}. First, we recall that $\Gamma =\Delta \kappa_{\mu}/2d\iota_1$ and $\beta = \kappa_{\mu}^{-1} (1 + 4\sqrt{d\iota_2} + \sqrt{2d\iota_3})$. We will first demonstrate that the selection of $\beta$ satisfy the requirement in Lemma \ref{lem:beta-gamma}. Recalling that $\lambda = B^{-2}$, through basic calculation, we have
\begin{align*}
\kappa_{\mu}\beta & \geq  1 + \sqrt{2d\log\big((1 + L^2B^2 16 d\Gamma^{-2}\iota_2)/ d\delta\big)} \\
& \geq 1 + \sqrt{2d\log\big((1 + L^2B^2 |\cC_T|)/ d\delta\big)} \\
& = \sqrt{\lambda}B + \sqrt{2d\log (\lambda + |\cC_{T}|L^2/d\lambda \delta)},
\end{align*}
where the first inequality holds by neglecting the positive term $4\sqrt{d\iota_2}$ and $d \geq 1$, the second inequality holds due to Lemma~\ref{lem:label-cplx} and the last equation holds by plugging in $\lambda=B^{-2}$. Now we come to the second statement. First, by basic computation, we have
\begin{equation*}
    \sqrt{2\iota_3} \leq \sqrt{2\log((1 + 16L^2B^2\Gamma^{-2}\iota_2)} + \sqrt{2\log (1/\delta))} .
\end{equation*}
Notice that we have $L \geq 1$, $B \geq 1$, and $\Gamma \leq 1$, which further implies that $LB\Gamma^{-1} \geq 1$, leading to
\begin{equation*}
    2 + 4\sqrt{\iota_2} \leq 6 \iota_2, \quad \sqrt{2\log((1 + 16L^2B^2\Gamma^{-2}\iota_2)} \leq 3 \iota_2.
\end{equation*}
Therefore, we have:
\begin{align*}
    2 + 4\sqrt{\iota_2} + \sqrt{2 \iota_3} & \leq 9 \iota_2 + 2\sqrt{\log (1/\delta)} \\
    & \leq 9 \log (6LB\sqrt{d}\Delta^{-1} \kappa_{\mu}^{-1} \iota_1) +  2\sqrt{\log (1/\delta)} .
\end{align*}
By Lemma \ref{lem:log2lin}, we can identify the sufficient condition for the following inequality
\begin{equation}\label{eq:log2lin}
    (6LB\sqrt{d}\Delta^{-1} \kappa_{\mu}^{-1})\iota_1 \geq 9 (6LB\sqrt{d}\Delta^{-1} \kappa_{\mu}^{-1}) \log (6LB\sqrt{d}\Delta^{-1} \kappa_{\mu}^{-1} \iota_1) +  2(6LB\sqrt{d}\Delta^{-1} \kappa_{\mu}^{-1})\sqrt{\log (1/\delta)} 
\end{equation}
is that
\begin{equation*}
    \iota_1 \geq 36 \log (108 LB\sqrt{d} \Delta^{-1} \kappa_{mu}^{-1} ) + \sqrt{8\log (1/\delta)},
\end{equation*}
which naturally holds due to our definition of $\iota_1$. Eliminating the $6LB\sqrt{d}\Delta^{-1} \kappa_{\mu}^{-1}$ term in~\eqref{eq:log2lin} yields that
\begin{equation*}
    \iota_1 \geq 2 + 4\sqrt{\iota_2} + \sqrt{2\iota_3},
\end{equation*}
which implies that 
\begin{equation*}
    2\beta\Gamma = \frac{\Delta \kappa_{\mu}}{2\sqrt{d}\iota_1} \frac{1}{\kappa_{\mu}} (1 + 2\sqrt{d\iota_2} + \sqrt{2\iota_3} ) < \Delta.
\end{equation*}
Thus, we complete the proof of Lemma \ref{lem:beta-gamma}.
\end{proof}
\subsection{Proof of Lemma \ref{optimistic}}
\begin{proof}[Proof of Lemma \ref{optimistic}]
    For each context $x\in \cX$ and action $y\in \cA$, we have
    \begin{align}
        |D_t(x,y)-\langle \hat{\btheta}_t, \bphi(x,y) - \bphi_t^2 \rangle|&=|\langle \hat{\btheta}_t- \btheta^*, \bphi(x,y) - \bphi_t^2 \rangle|\notag\\
        &\leq \|\hat{\btheta}_t - \btheta^* \|_{\bSigma_{t-1}^{-1}} \cdot \|\bphi(x,y) - \bphi_t^2\|_{\bSigma_{t-1}}\notag\\
        &\leq \beta  \|\bphi(x,y) - \bphi_t^2\|_{\bSigma_{t-1}},\label{eq:001}
    \end{align}
    where the first inequality holds due to Cauchy–Schwarz inequality and the second inequality holds due to event $\cE_1$. Therefore, we have
        \begin{align*}
        \langle \hat{\btheta}_t, \bphi(x,y) - \bphi_t^2 \rangle + \beta  \|\bphi(x,y) - \bphi_t^2\|_{\bSigma_{t-1}} \ge   D_t(x,y),
    \end{align*}
    where the first inequality holds due to \eqref{eq:001}. In addition, we have $D_t(x,y)=\langle\btheta^*, \bphi(x,y) - \bphi_t^2 \rangle\leq 2 $. Combing these two results, we have
    \begin{align*}
        \hat{D}_t(x,y)= \min\{\langle \hat{\btheta}_t, \bphi(x,y) - \bphi_t^2 \rangle + \beta  \|\bphi(x,y) - \bphi_t^2\|_{\bSigma_{t-1}},2\}\ge D_t(x,y).
    \end{align*}
    On the other hand, we have
    \begin{align*}
        \hat{D}_t(x,y)\leq \langle \hat{\btheta}_t, \bphi(x,y) - \bphi_t^2 \rangle + \beta  \|\bphi(x,y) - \bphi_t^2\|_{\bSigma_{t-1}} \leq  D_t(x,y) + 2\beta  \|\bphi(x,y) - \bphi_t^2\|_{\bSigma_{t-1}},
    \end{align*}
    where the first inequality holds due to the definition of $\hat{D}_t(x,y)$ and the second inequality holds due to \eqref{eq:001}. Thus, we complete the proof of Lemma \ref{optimistic}.
\end{proof}
\subsection{Proof of Lemma \ref{lem:omd}}
\begin{proof}[Proof of Lemma \ref{lem:omd}] The proof follows the approach in \citet{he2022near}. Recall that we assume the policy is updated in round $t$ according to the update rule~\eqref{eq:one-step-descent}, for all contexts $x \in \cX$. Thus, we have:
\begin{align}
    \exp\big\{ \eta \hat{D}_t(x, y)\big\}&=\frac{\pi_t(y|x) \exp\big\{\eta \hat{D}_t(x, y)\big\}}{\pi_t(y|x)}=\frac{\rho \pi_{t+1}(y|x)}{\pi_t(y|x)},\label{eq:31}
\end{align}
where $\rho = \sum_{y \in \cA} \pi_t(y|x) \exp\big\{ \eta \hat{D}_t(x, y) \big\}$ is the regularization term that is the same for all actions $y \in \cA$. Therefore, we have
\begin{align}
    &\sum_{y\in \cA} \eta \hat{D}_t(x, y)\big(\pi^*(y|x)-\pi_{t+1}(y|x)\big)\notag\\
    &=\sum_{y\in \cA} \big(\log \rho+\log \pi_{t+1}(y|x) -\log \pi_t(y|x) \big)\big(\pi^*(y|x)-\pi_{t+1}(y|x)\big)\notag\\
    &=\sum_{y\in \cA} \pi^*(y|x)\big(\log \pi_{t+1}(y|x) -\log \pi_t(y|x) \big) - \pi_{t+1}(y|x)\big(\log \pi_{t+1}(y|x) -\log \pi_t(y|x) \big)\notag\\
    &= \sum_{y\in \cA} \pi^*(y|x)\big(\log \pi^{*}(y|x) -\log \pi_{t}(y|x) \big) +\sum_{y\in \cA}\pi^*(y|x) \big(\log \pi_{t+1}(y|x) -\log \pi^*(y|x) \big)\notag\\
    &\qquad -\sum_{y\in \cA} \pi_{t+1}(y|x)\big(\log \pi_{t+1}(y|x) -\log \pi_t(y|x) \big)\notag\\
    &=\text{KL}\big(\pi^*(\cdot|x)\|\pi_{t}(\cdot|x)\big)-\text{KL}\big(\pi^*(\cdot|x)\|\pi_{t+1}(\cdot|x)\big)-\text{KL}\big(\pi_{t+1}(\cdot|x)\|\pi_t(\cdot|x)\big),\label{eq:32}
\end{align}
where the first equation holds due to \eqref{eq:31} and the second equation holds due to $\sum_{y\in \cA} \big(\pi^*(y|x)-\pi_{t+1}(y|x)\big) =0$.
Consequently, we have
\begin{align}
    &\EE_{y\sim \pi^*(\cdot|x)}\big[\hat{D}_t(x, y)\big]-\EE_{y\sim \pi_t(\cdot|x)}\big[\hat{D}_t(x, y)\big]\notag\\
    &=\sum_{y\in \cA} \hat{D}_t(x, y)\big(\pi^*(y|x)-\pi_t(y|x)\big)\notag\\
    &=\sum_{y\in \cA} \hat{D}_t(x, y)\big(\pi^*(y|x)-\pi_{t+1}(y|x)\big)+\sum_{y\in \cA} \hat{D}_t(x, y)\big(\pi_{t+1}(y|x)-\pi_t(y|x)\big)\notag\\
    &\leq \sum_{y\in \cA} \hat{D}_t(x, y)\big(\pi^*(y|x)-\pi_{t+1}(y|x)\big)+ 2\big\|\pi_{t+1}(\cdot|x)-\pi_t(\cdot|x)\big\|_1\notag\\
    &=\eta^{-1}\Big(\text{KL}\big(\pi^*(\cdot|x)\|\pi_{t}(\cdot|x)\big)-\text{KL}\big(\pi^*(\cdot|x)\|\pi_{t+1}(\cdot|x)\big)-\text{KL}\big(\pi_{t+1}(\cdot|x)\|\pi_t(\cdot|x)\big)\Big)\notag\\
    &\qquad + 2\big\|\pi_{t+1}(\cdot|x)-\pi_t(\cdot|x)\big\|_1\notag\\
    &\leq \eta^{-1}\Big(\text{KL}\big(\pi^*(\cdot|x)\|\pi_t(\cdot|x)\big)-\text{KL}\big(\pi^*(\cdot|x)\|\pi_{t+1}(\cdot|x)\big)\Big)\notag\\
    &\qquad + 2\big\|\pi_{t+1}(\cdot|x)-\pi_t(\cdot|x)\big\|_1-\frac{\big\|\pi_{t+1}(\cdot|x)-\pi_t(\cdot|x)\big\|_1^2}{2\eta}\notag\\
    &\leq 2 \eta +\eta^{-1}\Big(\text{KL}\big(\pi^*(\cdot|x)\|\pi_t(\cdot|x)\big)-\text{KL}\big(\pi^*(\cdot|x)\|\pi_{t+1}(\cdot|x)\big)\Big),
\end{align}
where the first inequality holds due to the fact that $0 \leq \hat{D}_t(x, y)\leq 2$, the second inequality holds due to Pinsker’s inequality and the last inequality holds due to the fact that $ax-bx^2\leq a^2/4b$. Finally, taking expectation over $x \sim \cD$ finishes the proof.
\end{proof}

\section{Auxiliary Lemmas}

\begin{lemma}[Lemma A.2, \citealp{shalev2014understanding}]\label{lem:log2lin}
Let $a \geq 1$ and $b \geq 0$, then $x \geq 4a \log(2a) + 2b$ results in $x \geq a \log x + b$.
\end{lemma}

\begin{lemma}[Theorem 1, \citealp{abbasi2011improved}]\label{lem:concentrate}
Let $\{\cF_t\}_{t=0}^\infty$ be a filtration. Let $\{\epsilon_t\}_{t=1}^\infty$ be a real-valued stochastic process such that $\epsilon_t$ is $\cF_t$-measurable and $\epsilon_t$ is conditionally $R$-sub-Gaussian for some $R \ge 0$. Let $\{\bphi_t\}_{t=1}^\infty$ be an $\RR^d$-valued stochastic process such that $\bphi_t$ is $\cF_{t-1}$ measurable and $\|\bphi_t\|_2 \le L$ for all $t$. For any $t \ge 0$, define $\Ub_t = \lambda \Ib + \sum_{i=1}^t\bphi_i\bphi_i^{\top}$. Then for any $\delta > 0$, with probability at least $1 - \delta$, for all $t \ge 0$, we have
\begin{align*}
    \left\|\sum_{i=1}^t\bphi_i\epsilon_i\right\|_{\Ub_t^{-1}}^2 \le 2R^2\log\left(\frac{\sqrt{\det(\Ub_t)}}{\sqrt{\det(\Ub_0)}\delta}\right).
\end{align*}
\end{lemma}

\begin{lemma}[Lemma 11,~\citealt{abbasi2011improved}]\label{lm:elliptical}
    Let $\{\bphi_i\}_{i=1}^t$ be a sequence in $\RR^d$, define $\Ub_i = \lambda \Ib + \sum_{i=1}^t\bphi_i\bphi_i^\top$, then 
    \begin{align*}
        \sum_{i=1}^t \min\left\{1, \|\bphi_i\|_{\Ub_{i - 1}^{-1}}^2\right\} \le 2d \log\left(\frac{\lambda d + tL^2}{\lambda d}\right).
    \end{align*}
\end{lemma}
The following auxiliary lemma and its corollary are useful
\begin{lemma}[Lemma A.2,~\citealt{shalev2014understanding}]\label{lm:xlogx}
Let $a \ge 1$ and $b > 0$. Then $x \ge 4a\log(2a) + 2b$ yields $x \ge a\log(x) + b$.
\end{lemma}

\begin{lemma}[Lemma C.7, \citealp{zhang2023interplay}]\label{lem:determine}
Suppose sequence $\{\xb_i\}_{i=1}^t \subset \RR^d$ and for any $i\leq t$, $\|\xb_i\|_2 \le L$. For any index subset $\cC \subseteq [t]$, define $\Ub = \lambda \Ib + \sum_{i \in \cC} \xb_i\xb_i^\top$ for some $\lambda > 0$, then $\det(\Ub) \le (\lambda + |\cC|L^2/d)^d$.
\end{lemma}

\begin{lemma}[Azuma–Hoeffding inequality, \citealt{cesa2006prediction}]\label{lemma:azuma}
Let $\{x_i\}_{i=1}^n$ be a martingale difference sequence with respect to a filtration $\{\cG_{i}\}$ satisfying $|x_i| \leq M$ for some constant $M$, $x_i$ is $\cG_{i+1}$-measurable, $\EE[x_i|\cG_i] = 0$. Then for any $0<\delta<1$, with probability at least $1-\delta$, we have 
\begin{align}
    \sum_{i=1}^n x_i\leq M\sqrt{2n \log (1/\delta)}.\notag
\end{align} 
\end{lemma}

\begin{lemma}[Lemma 11 in \cite{abbasi2011improved}]\label{Lemma:abba}
Let $\{\bphi_t\}_{t=1}^{+\infty}$ be a sequence in $\RR^d$, $\Ub$ a $d \times d$ positive definite matrix and define $\Ub_t=\Ub+\sum_{i=1}^{t} \bphi_i^{\top}\bphi_i$. If $\|\bphi_i\|_2\leq L$ and $\lambda_{\min}(\Ub)\ge \max(1,L^2),$ then we  have
\begin{align}
    \sum_{i=1}^{t} \bphi_i^{\top} (\Ub_{i-1})^{-1} \bphi_i\leq 2 \log \bigg(\frac{\det{\Ub_t}}{\det{\Ub}}\bigg).\notag
\end{align}
\end{lemma}

\bibliographystyle{icml2025}
\bibliography{example_paper}

\begin{thebibliography}{93}
\providecommand{\natexlab}[1]{#1}
\providecommand{\url}[1]{\texttt{#1}}
\expandafter\ifx\csname urlstyle\endcsname\relax
  \providecommand{\doi}[1]{doi: #1}\else
  \providecommand{\doi}{doi: \begingroup \urlstyle{rm}\Url}\fi

\bibitem[Abbasi-Yadkori et~al.(2011)Abbasi-Yadkori, P{\'a}l, and
  Szepesv{\'a}ri]{abbasi2011improved}
Abbasi-Yadkori, Y., P{\'a}l, D., and Szepesv{\'a}ri, C.
\newblock Improved algorithms for linear stochastic bandits.
\newblock \emph{Advances in neural information processing systems}, 24, 2011.

\bibitem[Achiam et~al.(2023)Achiam, Adler, Agarwal, Ahmad, Akkaya, Aleman,
  Almeida, Altenschmidt, Altman, Anadkat, et~al.]{achiam2023gpt}
Achiam, J., Adler, S., Agarwal, S., Ahmad, L., Akkaya, I., Aleman, F.~L.,
  Almeida, D., Altenschmidt, J., Altman, S., Anadkat, S., et~al.
\newblock Gpt-4 technical report.
\newblock \emph{arXiv preprint arXiv:2303.08774}, 2023.

\bibitem[Azar et~al.(2017)Azar, Osband, and Munos]{azar2017minimax}
Azar, M.~G., Osband, I., and Munos, R.
\newblock Minimax regret bounds for reinforcement learning.
\newblock In \emph{International conference on machine learning}, pp.\
  263--272. PMLR, 2017.

\bibitem[Bai et~al.(2022)Bai, Jones, Ndousse, Askell, Chen, DasSarma, Drain,
  Fort, Ganguli, Henighan, et~al.]{bai2022training}
Bai, Y., Jones, A., Ndousse, K., Askell, A., Chen, A., DasSarma, N., Drain, D.,
  Fort, S., Ganguli, D., Henighan, T., et~al.
\newblock Training a helpful and harmless assistant with reinforcement learning
  from human feedback.
\newblock \emph{arXiv preprint arXiv:2204.05862}, 2022.

\bibitem[Balcan et~al.(2006)Balcan, Beygelzimer, and
  Langford]{balcan2006agnostic}
Balcan, M.-F., Beygelzimer, A., and Langford, J.
\newblock Agnostic active learning.
\newblock In \emph{Proceedings of the 23rd international conference on Machine
  learning}, pp.\  65--72, 2006.

\bibitem[Balcan et~al.(2007)Balcan, Broder, and Zhang]{balcan2007margin}
Balcan, M.-F., Broder, A., and Zhang, T.
\newblock Margin based active learning.
\newblock In \emph{International Conference on Computational Learning Theory},
  pp.\  35--50. Springer, 2007.

\bibitem[Balsubramani et~al.(2016)Balsubramani, Karnin, Schapire, and
  Zoghi]{balsubramani2016instance}
Balsubramani, A., Karnin, Z., Schapire, R.~E., and Zoghi, M.
\newblock Instance-dependent regret bounds for dueling bandits.
\newblock In \emph{Conference on Learning Theory}, pp.\  336--360. PMLR, 2016.

\bibitem[Bengs et~al.(2022)Bengs, Saha, and
  H{\"u}llermeier]{bengs2022stochastic}
Bengs, V., Saha, A., and H{\"u}llermeier, E.
\newblock Stochastic contextual dueling bandits under linear stochastic
  transitivity models.
\newblock In \emph{International Conference on Machine Learning}, pp.\
  1764--1786. PMLR, 2022.

\bibitem[Bradley \& Terry(1952)Bradley and Terry]{bradley1952rank}
Bradley, R.~A. and Terry, M.~E.
\newblock Rank analysis of incomplete block designs: I. the method of paired
  comparisons.
\newblock \emph{Biometrika}, 39\penalty0 (3/4):\penalty0 324--345, 1952.

\bibitem[Cai et~al.(2020)Cai, Yang, Jin, and Wang]{cai2020provably}
Cai, Q., Yang, Z., Jin, C., and Wang, Z.
\newblock Provably efficient exploration in policy optimization.
\newblock In \emph{International Conference on Machine Learning}, pp.\
  1283--1294. PMLR, 2020.

\bibitem[Casper et~al.(2023)Casper, Davies, Shi, Gilbert, Scheurer, Rando,
  Freedman, Korbak, Lindner, Freire, et~al.]{casper2023open}
Casper, S., Davies, X., Shi, C., Gilbert, T.~K., Scheurer, J., Rando, J.,
  Freedman, R., Korbak, T., Lindner, D., Freire, P., et~al.
\newblock Open problems and fundamental limitations of reinforcement learning
  from human feedback.
\newblock \emph{arXiv preprint arXiv:2307.15217}, 2023.

\bibitem[Cesa-Bianchi \& Lugosi(2006)Cesa-Bianchi and
  Lugosi]{cesa2006prediction}
Cesa-Bianchi, N. and Lugosi, G.
\newblock \emph{Prediction, learning, and games}.
\newblock Cambridge university press, 2006.

\bibitem[Cesa-Bianchi et~al.(2005)Cesa-Bianchi, Lugosi, and
  Stoltz]{cesa2005minimizing}
Cesa-Bianchi, N., Lugosi, G., and Stoltz, G.
\newblock Minimizing regret with label efficient prediction.
\newblock \emph{IEEE Transactions on Information Theory}, 51\penalty0
  (6):\penalty0 2152--2162, 2005.

\bibitem[Cesa-Bianchi et~al.(2006)Cesa-Bianchi, Gentile, Zaniboni, and
  Warmuth]{cesa2006worst}
Cesa-Bianchi, N., Gentile, C., Zaniboni, L., and Warmuth, M.
\newblock Worst-case analysis of selective sampling for linear classification.
\newblock \emph{Journal of Machine Learning Research}, 7\penalty0 (7), 2006.

\bibitem[Cesa-Bianchi et~al.(2009)Cesa-Bianchi, Gentile, and
  Orabona]{cesa2009robust}
Cesa-Bianchi, N., Gentile, C., and Orabona, F.
\newblock Robust bounds for classification via selective sampling.
\newblock In \emph{Proceedings of the 26th annual international conference on
  machine learning}, pp.\  121--128, 2009.

\bibitem[Chiang et~al.(2023)Chiang, Li, Lin, Sheng, Wu, Zhang, Zheng, Zhuang,
  Zhuang, Gonzalez, Stoica, and Xing]{vicuna2023}
Chiang, W.-L., Li, Z., Lin, Z., Sheng, Y., Wu, Z., Zhang, H., Zheng, L.,
  Zhuang, S., Zhuang, Y., Gonzalez, J.~E., Stoica, I., and Xing, E.~P.
\newblock Vicuna: An open-source chatbot impressing gpt-4 with 90\%* chatgpt
  quality, March 2023.
\newblock URL \url{https://lmsys.org/blog/2023-03-30-vicuna/}.

\bibitem[Christiano et~al.(2017)Christiano, Leike, Brown, Martic, Legg, and
  Amodei]{christiano2017deep}
Christiano, P.~F., Leike, J., Brown, T., Martic, M., Legg, S., and Amodei, D.
\newblock Deep reinforcement learning from human preferences.
\newblock \emph{Advances in neural information processing systems}, 30, 2017.

\bibitem[Citovsky et~al.(2021)Citovsky, DeSalvo, Gentile, Karydas, Rajagopalan,
  Rostamizadeh, and Kumar]{citovsky2021batch}
Citovsky, G., DeSalvo, G., Gentile, C., Karydas, L., Rajagopalan, A.,
  Rostamizadeh, A., and Kumar, S.
\newblock Batch active learning at scale.
\newblock \emph{Advances in Neural Information Processing Systems},
  34:\penalty0 11933--11944, 2021.

\bibitem[Clark et~al.(2018)Clark, Cowhey, Etzioni, Khot, Sabharwal, Schoenick,
  and Tafjord]{clark2018think}
Clark, P., Cowhey, I., Etzioni, O., Khot, T., Sabharwal, A., Schoenick, C., and
  Tafjord, O.
\newblock Think you have solved question answering? try arc, the ai2 reasoning
  challenge.
\newblock \emph{arXiv preprint arXiv:1803.05457}, 2018.

\bibitem[Dasgupta(2005)]{dasgupta2005coarse}
Dasgupta, S.
\newblock Coarse sample complexity bounds for active learning.
\newblock \emph{Advances in neural information processing systems}, 18, 2005.

\bibitem[Dasgupta et~al.(2005)Dasgupta, Kalai, and
  Monteleoni]{dasgupta2005analysis}
Dasgupta, S., Kalai, A.~T., and Monteleoni, C.
\newblock Analysis of perceptron-based active learning.
\newblock In \emph{International conference on computational learning theory},
  pp.\  249--263. Springer, 2005.

\bibitem[Di et~al.(2023)Di, Jin, Wu, Zhao, Farnoud, and Gu]{di2023variance}
Di, Q., Jin, T., Wu, Y., Zhao, H., Farnoud, F., and Gu, Q.
\newblock Variance-aware regret bounds for stochastic contextual dueling
  bandits.
\newblock \emph{arXiv preprint arXiv:2310.00968}, 2023.

\bibitem[Ding et~al.(2023)Ding, Chen, Xu, Qin, Zheng, Hu, Liu, Sun, and
  Zhou]{ding2023enhancing}
Ding, N., Chen, Y., Xu, B., Qin, Y., Zheng, Z., Hu, S., Liu, Z., Sun, M., and
  Zhou, B.
\newblock Enhancing chat language models by scaling high-quality instructional
  conversations.
\newblock \emph{arXiv preprint arXiv:2305.14233}, 2023.

\bibitem[Dubois et~al.(2024)Dubois, Li, Taori, Zhang, Gulrajani, Ba, Guestrin,
  Liang, and Hashimoto]{dubois2024alpacafarm}
Dubois, Y., Li, C.~X., Taori, R., Zhang, T., Gulrajani, I., Ba, J., Guestrin,
  C., Liang, P.~S., and Hashimoto, T.~B.
\newblock Alpacafarm: A simulation framework for methods that learn from human
  feedback.
\newblock \emph{Advances in Neural Information Processing Systems}, 36, 2024.

\bibitem[Dud{\'\i}k et~al.(2015)Dud{\'\i}k, Hofmann, Schapire, Slivkins, and
  Zoghi]{dudik2015contextual}
Dud{\'\i}k, M., Hofmann, K., Schapire, R.~E., Slivkins, A., and Zoghi, M.
\newblock Contextual dueling bandits.
\newblock In \emph{Conference on Learning Theory}, pp.\  563--587. PMLR, 2015.

\bibitem[Falahatgar et~al.(2017)Falahatgar, Orlitsky, Pichapati, and
  Suresh]{falahatgar2017maximum}
Falahatgar, M., Orlitsky, A., Pichapati, V., and Suresh, A.~T.
\newblock Maximum selection and ranking under noisy comparisons.
\newblock In \emph{International Conference on Machine Learning}, pp.\
  1088--1096. PMLR, 2017.

\bibitem[Falahatgar et~al.(2018)Falahatgar, Jain, Orlitsky, Pichapati, and
  Ravindrakumar]{falahatgar2018limits}
Falahatgar, M., Jain, A., Orlitsky, A., Pichapati, V., and Ravindrakumar, V.
\newblock The limits of maxing, ranking, and preference learning.
\newblock In \emph{International conference on machine learning}, pp.\
  1427--1436. PMLR, 2018.

\bibitem[Filippi et~al.(2010)Filippi, Cappe, Garivier, and
  Szepesv{\'a}ri]{filippi2010parametric}
Filippi, S., Cappe, O., Garivier, A., and Szepesv{\'a}ri, C.
\newblock Parametric bandits: The generalized linear case.
\newblock \emph{Advances in Neural Information Processing Systems}, 23, 2010.

\bibitem[Gao et~al.(2023{\natexlab{a}})Gao, Zhao, Yu, and Xu]{gao2023exploring}
Gao, J., Zhao, H., Yu, C., and Xu, R.
\newblock Exploring the feasibility of chatgpt for event extraction.
\newblock \emph{arXiv preprint arXiv:2303.03836}, 2023{\natexlab{a}}.

\bibitem[Gao et~al.(2023{\natexlab{b}})Gao, Schulman, and
  Hilton]{gao2023scaling}
Gao, L., Schulman, J., and Hilton, J.
\newblock Scaling laws for reward model overoptimization.
\newblock In \emph{International Conference on Machine Learning}, pp.\
  10835--10866. PMLR, 2023{\natexlab{b}}.

\bibitem[Gentile et~al.(2022)Gentile, Wang, and Zhang]{gentile2022achieving}
Gentile, C., Wang, Z., and Zhang, T.
\newblock Achieving minimax rates in pool-based batch active learning.
\newblock In \emph{International Conference on Machine Learning}, pp.\
  7339--7367. PMLR, 2022.

\bibitem[Gu et~al.(2012)Gu, Zhang, Han, and Ding]{gu2012selective}
Gu, Q., Zhang, T., Han, J., and Ding, C.
\newblock Selective labeling via error bound minimization.
\newblock \emph{Advances in neural information processing systems}, 25, 2012.

\bibitem[Gu et~al.(2014)Gu, Zhang, and Han]{gu2014batch}
Gu, Q., Zhang, T., and Han, J.
\newblock Batch-mode active learning via error bound minimization.
\newblock In \emph{UAI}, pp.\  300--309, 2014.

\bibitem[Han et~al.(2023)Han, Peng, Yang, Wang, Liu, and
  Wan]{han2023information}
Han, R., Peng, T., Yang, C., Wang, B., Liu, L., and Wan, X.
\newblock Is information extraction solved by chatgpt? an analysis of
  performance, evaluation criteria, robustness and errors.
\newblock \emph{arXiv preprint arXiv:2305.14450}, 2023.

\bibitem[Hanneke \& Yang(2015)Hanneke and Yang]{hanneke2015minimax}
Hanneke, S. and Yang, L.
\newblock Minimax analysis of active learning.
\newblock \emph{J. Mach. Learn. Res.}, 16\penalty0 (1):\penalty0 3487--3602,
  2015.

\bibitem[Hanneke \& Yang(2021)Hanneke and Yang]{hanneke2021toward}
Hanneke, S. and Yang, L.
\newblock Toward a general theory of online selective sampling: Trading off
  mistakes and queries.
\newblock In \emph{International Conference on Artificial Intelligence and
  Statistics}, pp.\  3997--4005. PMLR, 2021.

\bibitem[He et~al.(2021)He, Zhou, and Gu]{he2021logarithmic}
He, J., Zhou, D., and Gu, Q.
\newblock Logarithmic regret for reinforcement learning with linear function
  approximation.
\newblock In \emph{International Conference on Machine Learning}, pp.\
  4171--4180. PMLR, 2021.

\bibitem[He et~al.(2022{\natexlab{a}})He, Zhou, and Gu]{he2022near}
He, J., Zhou, D., and Gu, Q.
\newblock Near-optimal policy optimization algorithms for learning adversarial
  linear mixture mdps.
\newblock In \emph{International Conference on Artificial Intelligence and
  Statistics}, pp.\  4259--4280. PMLR, 2022{\natexlab{a}}.

\bibitem[He et~al.(2022{\natexlab{b}})He, Zhou, Zhang, and Gu]{he2022nearly}
He, J., Zhou, D., Zhang, T., and Gu, Q.
\newblock Nearly optimal algorithms for linear contextual bandits with
  adversarial corruptions.
\newblock \emph{Advances in Neural Information Processing Systems},
  35:\penalty0 34614--34625, 2022{\natexlab{b}}.

\bibitem[He et~al.(2023)He, Zhao, Zhou, and Gu]{he2023nearly}
He, J., Zhao, H., Zhou, D., and Gu, Q.
\newblock Nearly minimax optimal reinforcement learning for linear markov
  decision processes.
\newblock In \emph{International Conference on Machine Learning}, pp.\
  12790--12822. PMLR, 2023.

\bibitem[Heckel et~al.(2018)Heckel, Simchowitz, Ramchandran, and
  Wainwright]{heckel2018approximate}
Heckel, R., Simchowitz, M., Ramchandran, K., and Wainwright, M.
\newblock Approximate ranking from pairwise comparisons.
\newblock In \emph{International Conference on Artificial Intelligence and
  Statistics}, pp.\  1057--1066. PMLR, 2018.

\bibitem[Hoi et~al.(2006)Hoi, Jin, Zhu, and Lyu]{hoi2006batch}
Hoi, S.~C., Jin, R., Zhu, J., and Lyu, M.~R.
\newblock Batch mode active learning and its application to medical image
  classification.
\newblock In \emph{Proceedings of the 23rd international conference on Machine
  learning}, pp.\  417--424, 2006.

\bibitem[Hu et~al.(2021)Hu, Wallis, Allen-Zhu, Li, Wang, Wang, Chen,
  et~al.]{hu2021lora}
Hu, E.~J., Wallis, P., Allen-Zhu, Z., Li, Y., Wang, S., Wang, L., Chen, W.,
  et~al.
\newblock Lora: Low-rank adaptation of large language models.
\newblock In \emph{International Conference on Learning Representations}, 2021.

\bibitem[Jamieson et~al.(2015)Jamieson, Katariya, Deshpande, and
  Nowak]{jamieson2015sparse}
Jamieson, K., Katariya, S., Deshpande, A., and Nowak, R.
\newblock Sparse dueling bandits.
\newblock In \emph{Artificial Intelligence and Statistics}, pp.\  416--424.
  PMLR, 2015.

\bibitem[Jiang et~al.(2023)Jiang, Sablayrolles, Mensch, Bamford, Chaplot,
  Casas, Bressand, Lengyel, Lample, Saulnier, et~al.]{jiang2023mistral}
Jiang, A.~Q., Sablayrolles, A., Mensch, A., Bamford, C., Chaplot, D.~S., Casas,
  D. d.~l., Bressand, F., Lengyel, G., Lample, G., Saulnier, L., et~al.
\newblock Mistral 7b.
\newblock \emph{arXiv preprint arXiv:2310.06825}, 2023.

\bibitem[Jin et~al.(2018)Jin, Allen-Zhu, Bubeck, and Jordan]{jin2018q}
Jin, C., Allen-Zhu, Z., Bubeck, S., and Jordan, M.~I.
\newblock Is q-learning provably efficient?
\newblock \emph{Advances in neural information processing systems}, 31, 2018.

\bibitem[Krueger et~al.(2020)Krueger, Leike, Evans, and
  Salvatier]{krueger2020active}
Krueger, D., Leike, J., Evans, O., and Salvatier, J.
\newblock Active reinforcement learning: Observing rewards at a cost.
\newblock \emph{arXiv preprint arXiv:2011.06709}, 2020.

\bibitem[Lee et~al.(2021{\natexlab{a}})Lee, Smith, and Abbeel]{lee2021pebble}
Lee, K., Smith, L., and Abbeel, P.
\newblock Pebble: Feedback-efficient interactive reinforcement learning via
  relabeling experience and unsupervised pre-training.
\newblock \emph{arXiv preprint arXiv:2106.05091}, 2021{\natexlab{a}}.

\bibitem[Lee et~al.(2021{\natexlab{b}})Lee, Smith, Dragan, and
  Abbeel]{lee2021b}
Lee, K., Smith, L., Dragan, A., and Abbeel, P.
\newblock B-pref: Benchmarking preference-based reinforcement learning.
\newblock \emph{arXiv preprint arXiv:2111.03026}, 2021{\natexlab{b}}.

\bibitem[Li et~al.(2017)Li, Lu, and Zhou]{li2017provably}
Li, L., Lu, Y., and Zhou, D.
\newblock Provably optimal algorithms for generalized linear contextual
  bandits.
\newblock In \emph{International Conference on Machine Learning}, pp.\
  2071--2080. PMLR, 2017.

\bibitem[Liang et~al.(2022)Liang, Shu, Lee, and Abbeel]{liang2022reward}
Liang, X., Shu, K., Lee, K., and Abbeel, P.
\newblock Reward uncertainty for exploration in preference-based reinforcement
  learning.
\newblock \emph{arXiv preprint arXiv:2205.12401}, 2022.

\bibitem[Lin et~al.(2021)Lin, Hilton, and Evans]{lin2021truthfulqa}
Lin, S., Hilton, J., and Evans, O.
\newblock Truthfulqa: Measuring how models mimic human falsehoods.
\newblock \emph{arXiv preprint arXiv:2109.07958}, 2021.

\bibitem[Liu et~al.(2023)Liu, Zeng, He, Jiang, and He]{liu2023makes}
Liu, W., Zeng, W., He, K., Jiang, Y., and He, J.
\newblock What makes good data for alignment? a comprehensive study of
  automatic data selection in instruction tuning.
\newblock \emph{arXiv preprint arXiv:2312.15685}, 2023.

\bibitem[Lou et~al.(2022)Lou, Jin, Wu, Xu, Gu, and Farnoud]{lou2022active}
Lou, H., Jin, T., Wu, Y., Xu, P., Gu, Q., and Farnoud, F.
\newblock Active ranking without strong stochastic transitivity.
\newblock \emph{Advances in neural information processing systems},
  35:\penalty0 297--309, 2022.

\bibitem[Mehta et~al.(2023)Mehta, Das, Neopane, Dai, Bogunovic, Schneider, and
  Neiswanger]{mehta2023sample}
Mehta, V., Das, V., Neopane, O., Dai, Y., Bogunovic, I., Schneider, J., and
  Neiswanger, W.
\newblock Sample efficient reinforcement learning from human feedback via
  active exploration.
\newblock \emph{arXiv preprint arXiv:2312.00267}, 2023.

\bibitem[Muldrew et~al.(2024)Muldrew, Hayes, Zhang, and
  Barber]{muldrew2024active}
Muldrew, W., Hayes, P., Zhang, M., and Barber, D.
\newblock Active preference learning for large language models.
\newblock \emph{arXiv preprint arXiv:2402.08114}, 2024.

\bibitem[Munos et~al.(2023)Munos, Valko, Calandriello, Azar, Rowland, Guo,
  Tang, Geist, Mesnard, Michi, et~al.]{munos2023nash}
Munos, R., Valko, M., Calandriello, D., Azar, M.~G., Rowland, M., Guo, Z.~D.,
  Tang, Y., Geist, M., Mesnard, T., Michi, A., et~al.
\newblock Nash learning from human feedback.
\newblock \emph{arXiv preprint arXiv:2312.00886}, 2023.

\bibitem[Ouyang et~al.(2022)Ouyang, Wu, Jiang, Almeida, Wainwright, Mishkin,
  Zhang, Agarwal, Slama, Ray, et~al.]{ouyang2022training}
Ouyang, L., Wu, J., Jiang, X., Almeida, D., Wainwright, C., Mishkin, P., Zhang,
  C., Agarwal, S., Slama, K., Ray, A., et~al.
\newblock Training language models to follow instructions with human feedback.
\newblock \emph{Advances in Neural Information Processing Systems},
  35:\penalty0 27730--27744, 2022.

\bibitem[Rafailov et~al.(2024)Rafailov, Sharma, Mitchell, Manning, Ermon, and
  Finn]{rafailov2024direct}
Rafailov, R., Sharma, A., Mitchell, E., Manning, C.~D., Ermon, S., and Finn, C.
\newblock Direct preference optimization: Your language model is secretly a
  reward model.
\newblock \emph{Advances in Neural Information Processing Systems}, 36, 2024.

\bibitem[Ramamohan et~al.(2016)Ramamohan, Rajkumar, and
  Agarwal]{ramamohan2016dueling}
Ramamohan, S.~Y., Rajkumar, A., and Agarwal, S.
\newblock Dueling bandits: Beyond condorcet winners to general tournament
  solutions.
\newblock \emph{Advances in Neural Information Processing Systems}, 29, 2016.

\bibitem[Ren et~al.(2019)Ren, Liu, and Shroff]{ren2019sample}
Ren, W., Liu, J.~K., and Shroff, N.
\newblock On sample complexity upper and lower bounds for exact ranking from
  noisy comparisons.
\newblock \emph{Advances in Neural Information Processing Systems}, 32, 2019.

\bibitem[Saha(2021)]{saha2021optimal}
Saha, A.
\newblock Optimal algorithms for stochastic contextual preference bandits.
\newblock \emph{Advances in Neural Information Processing Systems},
  34:\penalty0 30050--30062, 2021.

\bibitem[Saha \& Gaillard(2022)Saha and Gaillard]{saha2022versatile}
Saha, A. and Gaillard, P.
\newblock Versatile dueling bandits: Best-of-both world analyses for learning
  from relative preferences.
\newblock In \emph{International Conference on Machine Learning}, pp.\
  19011--19026. PMLR, 2022.

\bibitem[Saha \& Krishnamurthy(2022)Saha and Krishnamurthy]{saha2022efficient}
Saha, A. and Krishnamurthy, A.
\newblock Efficient and optimal algorithms for contextual dueling bandits under
  realizability.
\newblock In \emph{International Conference on Algorithmic Learning Theory},
  pp.\  968--994. PMLR, 2022.

\bibitem[Schulman et~al.(2017)Schulman, Wolski, Dhariwal, Radford, and
  Klimov]{schulman2017proximal}
Schulman, J., Wolski, F., Dhariwal, P., Radford, A., and Klimov, O.
\newblock Proximal policy optimization algorithms.
\newblock \emph{arXiv preprint arXiv:1707.06347}, 2017.

\bibitem[Schulze \& Evans(2018)Schulze and Evans]{schulze2018active}
Schulze, S. and Evans, O.
\newblock Active reinforcement learning with monte-carlo tree search.
\newblock \emph{arXiv preprint arXiv:1803.04926}, 2018.

\bibitem[Sekhari et~al.(2024)Sekhari, Sridharan, Sun, and
  Wu]{sekhari2024contextual}
Sekhari, A., Sridharan, K., Sun, W., and Wu, R.
\newblock Contextual bandits and imitation learning with preference-based
  active queries.
\newblock \emph{Advances in Neural Information Processing Systems}, 36, 2024.

\bibitem[Shalev-Shwartz \& Ben-David(2014)Shalev-Shwartz and
  Ben-David]{shalev2014understanding}
Shalev-Shwartz, S. and Ben-David, S.
\newblock \emph{Understanding machine learning: From theory to algorithms}.
\newblock Cambridge university press, 2014.

\bibitem[Simchowitz \& Jamieson(2019)Simchowitz and
  Jamieson]{simchowitz2019non}
Simchowitz, M. and Jamieson, K.~G.
\newblock Non-asymptotic gap-dependent regret bounds for tabular mdps.
\newblock In \emph{Advances in Neural Information Processing Systems}, pp.\
  1153--1162, 2019.

\bibitem[Team et~al.(2024)Team, Mesnard, Hardin, Dadashi, Bhupatiraju, Pathak,
  Sifre, Rivi{\`e}re, Kale, Love, et~al.]{team2024gemma}
Team, G., Mesnard, T., Hardin, C., Dadashi, R., Bhupatiraju, S., Pathak, S.,
  Sifre, L., Rivi{\`e}re, M., Kale, M.~S., Love, J., et~al.
\newblock Gemma: Open models based on gemini research and technology.
\newblock \emph{arXiv preprint arXiv:2403.08295}, 2024.

\bibitem[Touvron et~al.(2023)Touvron, Martin, Stone, Albert, Almahairi, Babaei,
  Bashlykov, Batra, Bhargava, Bhosale, et~al.]{touvron2023llama}
Touvron, H., Martin, L., Stone, K., Albert, P., Almahairi, A., Babaei, Y.,
  Bashlykov, N., Batra, S., Bhargava, P., Bhosale, S., et~al.
\newblock Llama 2: Open foundation and fine-tuned chat models.
\newblock \emph{arXiv preprint arXiv:2307.09288}, 2023.

\bibitem[Tucker et~al.(2023)Tucker, Biddulph, Wang, and
  Joachims]{tucker2023bandits}
Tucker, A.~D., Biddulph, C., Wang, C., and Joachims, T.
\newblock Bandits with costly reward observations.
\newblock In \emph{Uncertainty in Artificial Intelligence}, pp.\  2147--2156.
  PMLR, 2023.

\bibitem[Tunstall et~al.(2023)Tunstall, Beeching, Lambert, Rajani, Rasul,
  Belkada, Huang, von Werra, Fourrier, Habib, Sarrazin, Sanseviero, Rush, and
  Wolf]{tunstall2023zephyr}
Tunstall, L., Beeching, E., Lambert, N., Rajani, N., Rasul, K., Belkada, Y.,
  Huang, S., von Werra, L., Fourrier, C., Habib, N., Sarrazin, N., Sanseviero,
  O., Rush, A.~M., and Wolf, T.
\newblock Zephyr: Direct distillation of lm alignment, 2023.

\bibitem[Wang et~al.(2023)Wang, Liu, and Jin]{wang2023rlhf}
Wang, Y., Liu, Q., and Jin, C.
\newblock Is rlhf more difficult than standard rl? a theoretical perspective.
\newblock \emph{Advances in Neural Information Processing Systems},
  36:\penalty0 76006--76032, 2023.

\bibitem[Wei et~al.(2023)Wei, Cui, Cheng, Wang, Zhang, Huang, Xie, Xu, Chen,
  Zhang, et~al.]{wei2023zero}
Wei, X., Cui, X., Cheng, N., Wang, X., Zhang, X., Huang, S., Xie, P., Xu, J.,
  Chen, Y., Zhang, M., et~al.
\newblock Zero-shot information extraction via chatting with chatgpt.
\newblock \emph{arXiv preprint arXiv:2302.10205}, 2023.

\bibitem[Wirth et~al.(2017)Wirth, Akrour, Neumann, F{\"u}rnkranz,
  et~al.]{wirth2017survey}
Wirth, C., Akrour, R., Neumann, G., F{\"u}rnkranz, J., et~al.
\newblock A survey of preference-based reinforcement learning methods.
\newblock \emph{Journal of Machine Learning Research}, 18\penalty0
  (136):\penalty0 1--46, 2017.

\bibitem[Wu \& Liu(2016)Wu and Liu]{wu2016double}
Wu, H. and Liu, X.
\newblock Double thompson sampling for dueling bandits.
\newblock \emph{Advances in neural information processing systems}, 29, 2016.

\bibitem[Wu \& Sun(2023)Wu and Sun]{wu2023making}
Wu, R. and Sun, W.
\newblock Making rl with preference-based feedback efficient via randomization.
\newblock \emph{arXiv preprint arXiv:2310.14554}, 2023.

\bibitem[Wu et~al.(2022)Wu, Jin, Lou, Xu, Farnoud, and Gu]{wu2022adaptive}
Wu, Y., Jin, T., Lou, H., Xu, P., Farnoud, F., and Gu, Q.
\newblock Adaptive sampling for heterogeneous rank aggregation from noisy
  pairwise comparisons.
\newblock In \emph{International Conference on Artificial Intelligence and
  Statistics}, pp.\  11014--11036. PMLR, 2022.

\bibitem[Wu et~al.(2023)Wu, Jin, Lou, Farnoud, and Gu]{wu2023borda}
Wu, Y., Jin, T., Lou, H., Farnoud, F., and Gu, Q.
\newblock Borda regret minimization for generalized linear dueling bandits.
\newblock \emph{arXiv preprint arXiv:2303.08816}, 2023.

\bibitem[Xiong et~al.(2023)Xiong, Dong, Ye, Zhong, Jiang, and
  Zhang]{xiong2023gibbs}
Xiong, W., Dong, H., Ye, C., Zhong, H., Jiang, N., and Zhang, T.
\newblock Gibbs sampling from human feedback: A provable kl-constrained
  framework for rlhf.
\newblock \emph{arXiv preprint arXiv:2312.11456}, 2023.

\bibitem[Yang et~al.(2020)Yang, Yang, and Du]{yang2020q}
Yang, K., Yang, L.~F., and Du, S.~S.
\newblock $ q $-learning with logarithmic regret.
\newblock \emph{arXiv preprint arXiv:2006.09118}, 2020.

\bibitem[Yuan et~al.(2023)Yuan, Xie, and Ananiadou]{yuan2023zero}
Yuan, C., Xie, Q., and Ananiadou, S.
\newblock Zero-shot temporal relation extraction with chatgpt.
\newblock \emph{arXiv preprint arXiv:2304.05454}, 2023.

\bibitem[Yue et~al.(2012)Yue, Broder, Kleinberg, and Joachims]{yue2012k}
Yue, Y., Broder, J., Kleinberg, R., and Joachims, T.
\newblock The k-armed dueling bandits problem.
\newblock \emph{Journal of Computer and System Sciences}, 78\penalty0
  (5):\penalty0 1538--1556, 2012.

\bibitem[Zellers et~al.(2019)Zellers, Holtzman, Bisk, Farhadi, and
  Choi]{zellers2019hellaswag}
Zellers, R., Holtzman, A., Bisk, Y., Farhadi, A., and Choi, Y.
\newblock Hellaswag: Can a machine really finish your sentence?
\newblock \emph{arXiv preprint arXiv:1905.07830}, 2019.

\bibitem[Zhan et~al.(2023)Zhan, Uehara, Sun, and Lee]{zhan2023query}
Zhan, W., Uehara, M., Sun, W., and Lee, J.~D.
\newblock How to query human feedback efficiently in rl?
\newblock 2023.

\bibitem[Zhang \& Oles(2000)Zhang and Oles]{zhang2000value}
Zhang, T. and Oles, F.
\newblock The value of unlabeled data for classification problems.
\newblock In \emph{Proceedings of the Seventeenth International Conference on
  Machine Learning,(Langley, P., ed.)}, volume~20, pp.\ ~0. Citeseer, 2000.

\bibitem[Zhang et~al.(2023)Zhang, He, Fan, and Gu]{zhang2023interplay}
Zhang, W., He, J., Fan, Z., and Gu, Q.
\newblock On the interplay between misspecification and sub-optimality gap in
  linear contextual bandits.
\newblock \emph{arXiv preprint arXiv:2303.09390}, 2023.

\bibitem[Zhao et~al.(2023)Zhao, He, Zhou, Zhang, and Gu]{zhao2023variance}
Zhao, H., He, J., Zhou, D., Zhang, T., and Gu, Q.
\newblock Variance-dependent regret bounds for linear bandits and reinforcement
  learning: Adaptivity and computational efficiency.
\newblock \emph{arXiv preprint arXiv:2302.10371}, 2023.

\bibitem[Zheng et~al.(2024)Zheng, Chiang, Sheng, Zhuang, Wu, Zhuang, Lin, Li,
  Li, Xing, et~al.]{zheng2024judging}
Zheng, L., Chiang, W.-L., Sheng, Y., Zhuang, S., Wu, Z., Zhuang, Y., Lin, Z.,
  Li, Z., Li, D., Xing, E., et~al.
\newblock Judging llm-as-a-judge with mt-bench and chatbot arena.
\newblock \emph{Advances in Neural Information Processing Systems}, 36, 2024.

\bibitem[Zhou \& Gu(2022)Zhou and Gu]{zhou2022computationally}
Zhou, D. and Gu, Q.
\newblock Computationally efficient horizon-free reinforcement learning for
  linear mixture mdps.
\newblock \emph{Advances in neural information processing systems},
  35:\penalty0 36337--36349, 2022.

\bibitem[Zhu et~al.(2023)Zhu, Jordan, and Jiao]{zhu2023principled}
Zhu, B., Jordan, M., and Jiao, J.
\newblock Principled reinforcement learning with human feedback from pairwise
  or k-wise comparisons.
\newblock In \emph{International Conference on Machine Learning}, pp.\
  43037--43067. PMLR, 2023.

\bibitem[Ziegler et~al.(2019)Ziegler, Stiennon, Wu, Brown, Radford, Amodei,
  Christiano, and Irving]{ziegler2019fine}
Ziegler, D.~M., Stiennon, N., Wu, J., Brown, T.~B., Radford, A., Amodei, D.,
  Christiano, P., and Irving, G.
\newblock Fine-tuning language models from human preferences.
\newblock \emph{arXiv preprint arXiv:1909.08593}, 2019.

\end{thebibliography}

\end{document}